\documentclass{article}

\usepackage{microtype}
\usepackage{graphicx}
\usepackage{subcaption}
\usepackage{booktabs} 
\usepackage{enumitem}
\usepackage{mathtools}         
  \mathtoolsset{showonlyrefs}
\usepackage{makecell}

\usepackage{hyperref}

\newcommand{\TC}{\mathcal{T}_{c_0}}

\newcommand{\iid}{\stackrel{\text{iid}}{\sim}}

\newcommand{\R}{\mathbb{R}}
\newcommand{\E}{\mathbb{E}}

\newcommand{\RF}{\mathrm{RF}}

\usepackage[preprint]{icml2026}

\usepackage{amsmath}
\usepackage{amssymb}
\usepackage{mathtools}
\usepackage{amsthm}

\usepackage[capitalize,noabbrev]{cleveref}

\theoremstyle{plain}
\newtheorem{theorem}{Theorem}[section]
\newtheorem{proposition}[theorem]{Proposition}
\newtheorem{lemma}[theorem]{Lemma}
\newtheorem{corollary}[theorem]{Corollary}
\newtheorem{example}[theorem]{Example}
\theoremstyle{definition}
\newtheorem{definition}[theorem]{Definition}
\newtheorem{assumption}[theorem]{Assumption}
\theoremstyle{remark}
\newtheorem{remark}[theorem]{Remark}

\usepackage[textsize=tiny]{todonotes}

\icmltitlerunning{Copula‐Stein Discrepancy}

\begin{document}

\twocolumn[
\icmltitle{Copula‐Stein Discrepancy: A Generator‐Based Stein Operator for Archimedean Dependence}



  \icmlsetsymbol{equal}{*}

 \begin{icmlauthorlist}

\icmlauthor{Agnideep Aich}{ull}
\icmlauthor{Ashit Baran Aich}{ind}

\end{icmlauthorlist}

\icmlaffiliation{ull}{Department of Mathematics, University of Louisiana at Lafayette, Lafayette, Louisiana, USA}
\icmlaffiliation{ind}{Department of Statistics, formerly of Presidency College, Kolkata, India}

\icmlcorrespondingauthor{Agnideep Aich}{agnideep.aich1@louisiana.edu}
\icmlcorrespondingauthor{Ashit Baran Aich}{aichnsou@gmail.com}

  \icmlkeywords{Machine Learning, ICML}

  \vskip 0.3in
]



\printAffiliationsAndNotice{}  

\begin{abstract}
  Kernel Stein discrepancies (KSDs) are widely used for goodness-of-fit testing, but standard KSDs can be insensitive to higher-order dependence features such as tail dependence. We introduce the Copula-Stein Discrepancy (CSD), which defines a Stein operator directly on the copula density to target dependence geometry rather than the joint score. For Archimedean copulas, CSD admits a closed-form Stein kernel derived from the scalar generator. We prove that CSD metrizes weak convergence of copula distributions, admits an empirical estimator with minimax-optimal rate $O_P(n^{-1/2})$, and is sensitive to differences in tail dependence coefficients. We further extend the framework beyond Archimedean families to general copulas, including elliptical and vine constructions. Computationally, exact CSD kernel evaluation is linear in dimension, and a random-feature approximation reduces the quadratic $O(n^2)$ sample scaling to near-linear $\tilde{O}(n)$; experiments show near-nominal Type~I error, increasing power, and rapid concentration of the approximation toward the exact $\widehat{\mathrm{CSD}}_n^2$ as the number of features grows.
\end{abstract}

\section{Introduction}
\label{sec:introduction}

The problem of measuring the discrepancy between a generative model and a set of observed data is fundamental to statistical inference and machine learning. Kernel Stein Discrepancies (KSDs) have emerged as a powerful class of tools for such goodness-of-fit (GoF) testing, particularly in high-dimensional settings where normalizing constants are intractable \citep{LiuLeeJordan2016, ChwialkowskiStrathmannGretton2016}. These methods construct a ``Stein operator" based on the score of the target distribution (i.e., the gradient of its log-density) and leverage the theory of reproducing kernel Hilbert spaces (RKHS) to define a metric for sample quality \citep{GorhamMackey2017}.

However, a critical limitation of standard KSDs is that they operate on the score of the joint probability density. This often renders them insensitive to the rich, higher-order dependence structures that are crucial in many domains. For instance, in finance and hydrology, properly modeling the tail dependence, the propensity for multiple variables to experience extreme events simultaneously, is paramount, yet this is a feature that joint-density-based methods can easily miss \citep{BreymannDiasEmbrechts2003, Patton2012}. Indeed, recent work has shown that standard KSDs can fail to detect non-converging samples precisely because their kernels decay too rapidly in the tails \citep{GorhamMackey2017}, a problem our dependence-focused approach is designed to address.

This paper addresses this gap by proposing a new paradigm for Stein-based GoF testing that is explicitly dependence-aware. We achieve this by building our discrepancy on the principles of copula theory. A copula is a function that isolates the dependence structure of a multivariate distribution from its marginal distributions \citep{Sklar1959, Nelsen2006}. By defining a Stein operator directly on the \emph{copula density} rather than the joint density, we introduce the Copula-Stein Discrepancy (CSD), a novel metric designed to be maximally sensitive to misspecifications in the dependence structure. For the broad and flexible class of Archimedean copulas, our approach is particularly powerful, yielding a closed-form Stein kernel derived directly from the copula's scalar generator function \citep{GenestRivest1993}.

Our contributions are as follows:
\begin{enumerate}
    \item We introduce the Copula-Stein Discrepancy (CSD) and derive a new Stein operator and corresponding closed-form kernel tailored to copula densities.
    \item We provide a comprehensive theoretical analysis, proving that CSD (i) metrizes weak convergence of copula distributions; (ii) has an empirical estimator that converges at the minimax optimal rate of $O_P(n^{-1/2})$; and (iii) is provably sensitive to differences in tail dependence coefficients.
    \item We demonstrate the generality of the CSD framework by extending it beyond Archimedean families to any copula satisfying mild regularity conditions, including important classes like elliptical and vine copulas.
    \item We establish the practical viability of CSD by proving its exact computation scales linearly in dimension and developing a random feature approximation that reduces the sample complexity from quadratic to nearly linear, enabling its use on large-scale datasets.
    \item We provide numerical experiments demonstrating near-nominal Type~I error and increasing power for the CSD test in a bivariate setting, and validating the accuracy of the random-feature approximation against the exact $\widehat{\mathrm{CSD}}_n^2$.
\end{enumerate}
The remainder of the paper is structured as follows. Section \ref{sec:related_work} discusses related work. Section~\ref{sec:background} provides background on KSDs and copulas. Section~\ref{sec:notation} introduces all the notations and assumptions used throughout the paper. Section \ref{sec:csd} formally introduces the CSD, its operator, and kernel. Section \ref{sec:theory} establishes its key theoretical properties. Section \ref{sec:beyond_archimedean} extends the framework to general copulas, and Section \ref{sec:computation} details its computational properties and scalable algorithms. Finally, Section~\ref{sec:experiments} reports numerical experiments illustrating finite-sample testing performance and the accuracy of the random-feature approximation.

\section{Related Work}
\label{sec:related_work}

Our work builds upon and contributes to two major lines of research: kernel-based goodness-of-fit testing and the classical literature on copula diagnostics.

\subsection{Kernel Stein Discrepancies}
The foundation of our method is the Kernel Stein Discrepancy (KSD), introduced by \citet{LiuLeeJordan2016} and \citet{ChwialkowskiStrathmannGretton2016}. These works combined the principles of Stein's method \citep{BarbourChen2005} with the power of RKHSs to create GoF tests that do not require knowledge of the target's normalizing constant. The framework has since been extended to measure MCMC sample quality \citep{GorhamMackey2017} and to compare models with latent variables \citep{LiuEtAl2019}. The primary innovation of CSD is its target of analysis. Whereas existing KSDs operate on the score of the joint density $p(x_1, \dots, x_d)$, CSD operates on the score of the copula density $c(u_1, \dots, u_d)$. This fundamental shift allows CSD to isolate and test the dependence structure, a feature to which standard KSDs are often insensitive. Furthermore, while the quadratic complexity of KSDs is a known challenge, we follow the line of work on scalable kernel methods \citep{RahimiRecht2007, Jitkrittum2017} to develop a nearly linear-time random feature approximation specifically for our Copula-Stein kernel.

\subsection{Goodness-of-Fit Testing for Copulas}
There is a rich classical literature on GoF testing for copulas, which largely relies on statistics derived from the empirical copula process \citep{GenestRemillard2004}. Many of these are ``blanket" or omnibus tests that compare the empirical copula to a hypothesized parametric form using Cramér-von Mises or Kolmogorov-Smirnov type statistics \citep{GenestRemillardBeaudoin2009, Fermanian2012}. Other notable approaches include tests based on positive-definite bilinear forms \citep{Panchenko2005} and the information matrix equality \citep{HuangProkhorov2014}.

CSD contributes a new class of GoF tests to this literature with several key advantages. As a discrepancy metric, its value provides a quantitative measure of misspecification, whereas many classical tests only yield a p-value. Most importantly, CSD is specifically designed to be sensitive to features like tail dependence, a known challenge highlighted in financial applications \citep{MalevergneSornette2003}. While specialized tests for extreme-value dependence exist \citep{KojadinovicYan2011, GenestSegers2009}, CSD provides a general framework with theoretical guarantees on its sensitivity to tail coefficients for any copula family.

\subsection{Copula Modeling and Diagnostics}
Our work also serves as a diagnostic tool for the vast literature on copula modeling. Since the seminal work on rank-based inference \citep{GenestGhoudiRivest1995}, copulas have been used to build sophisticated semiparametric models in econometrics and time series analysis \citep{ChenFan2006, ChenFanTsyrennikov2006, Patton2006}. The development of highly flexible models, such as vine copulas \citep{CzadoSchepsmeierMin2012}, has further increased the need for reliable diagnostic tools and GoF tests \citep{SchepsmeierCzado2011}. CSD provides a principled and scalable diagnostic to assess the fit of these complex models, complementing existing estimation techniques by providing a rigorous check on the specified dependence structure.

\section{Background: Copulas and Stein's Method}
\label{sec:background}

Our work builds a bridge between two distinct areas of statistical methodology: the theory of copula-based dependence modeling and kernel-based goodness-of-fit testing using Stein's method. We briefly review the essential concepts from both fields.

\subsection{Copulas and Dependence Modeling}
Copulas are functions that separate a multivariate distribution into its marginal distributions and its dependence structure. By Sklar's Theorem, any joint cumulative distribution function (CDF) $F$ with marginal CDFs $F_1, \dots, F_d$ can be written as $F(x_1, \dots, x_d) = C(F_1(x_1), \dots, F_d(x_d))$, where $C: [0,1]^d \to [0,1]$ is a unique copula \citep{Sklar1959}. This allows the study of dependence on the unit hypercube, free from the influence of the marginals.

A particularly useful and broad class are Archimedean copulas, which are constructed from a single univariate function $\varphi: [0, 1] \to [0, \infty]$ called a generator. The generator must be continuous, strictly decreasing, and convex, with $\varphi(1)=0$. The copula is then defined via its pseudo-inverse $\varphi^{[-1]}$ as:
\begin{align}
C(u_1, \dots, u_d) = \varphi^{[-1]}\left(\sum_{j=1}^d \varphi(u_j)\right).
\end{align}
This parsimonious construction can model a wide variety of dependence patterns, including different forms of tail dependence, simply by changing the generator $\varphi$. For instance, the Clayton copula exhibits lower tail dependence, the Gumbel copula exhibits upper tail dependence, and the Frank copula exhibits none \citep{Nelsen2006, Joe2014}. The existence of an analytic generator is the key property we will exploit to construct our novel Stein operator.

\subsection{Kernel Stein Discrepancy}
Stein's method \citep{stein72bound} provides a powerful recipe for constructing metrics between probability distributions \citep{BarbourChen2005}. The core idea is to find a \emph{Stein operator}, $\mathcal{T}_{p_0}$, associated with a target density $p_0$, such that for any suitable test function $f$, the operator has zero expectation under the target: $\E_{X \sim p_0}[\mathcal{T}_{p_0} f(X)] = 0$. The Kernel Stein Discrepancy (KSD) defines a metric by choosing the test functions from the unit ball of an RKHS and solving for the supremum, which results in a closed-form kernel test \citep{LiuLeeJordan2016, ChwialkowskiStrathmannGretton2016}. Our work adapts this powerful KSD framework from the domain of joint densities to the domain of copula densities.


\section{Standing Assumptions and Notation}
\label{sec:notation}

This section collects the standing regularity conditions used throughout the paper and fixes notation.
Unless stated otherwise, all theoretical results are proved under Assumption~\ref{assum:unified} (A1--A5).

\textbf{Distributions and Samples.}
A probability distribution is $P$ with density $p$. The target (null) copula distribution is $C_0$ with density $c_0$; a generic copula is $C$ with density $c$. The $d$–dimensional unit hypercube is $[0,1]^d$. An i.i.d.\ sample of size $n$ from a copula is $\{U_i\}_{i=1}^n$, $U_i\in[0,1]^d$. The empirical measure is $C_n=\frac1n\sum_{i=1}^n\delta_{U_i}$. Expectations are $\E[\cdot]$.

\textbf{Copula Framework.}
A copula is $C(u)$ with $u=(u_1,\dots,u_d)\in[0,1]^d$. For Archimedean copulas, $\varphi:(0,1]\to[0,\infty)$ is the generator, $\psi=\varphi^{-1}$, and $\varphi^{(k)}$ denotes the $k$th derivative. Parametric dependence parameters are denoted $\theta$. The copula score is $s(u)=\nabla\log c(u)$. Lower/upper tail coefficients are $\lambda_L,\lambda_U$. For Gaussian (elliptical) copulas, $\Sigma$ is the correlation matrix and $\Phi^{-1}$ the standard normal quantile.

\textbf{RKHS/Stein Framework.}
The base kernel is $k(u,v)$ on $[0,1]^d\times[0,1]^d$ with scalar RKHS $\mathcal H$ and vector RKHS $\mathcal H^d$. A vector test function is $g(u)\in\mathcal H^d$. The Copula-Stein operator at target $C_0$ is
\begin{align}
\TC g(u)=\sum_{j=1}^d\!\Big[\partial_{u_j}g_j(u)+g_j(u)\,\partial_{u_j}\log c_0(u)\Big].
\end{align}
The witness $\xi_{C_0}(u)\in\mathcal H^d$ satisfies $\TC g(u)=\langle g,\xi_{C_0}(u)\rangle_{\mathcal H^d}$. The Copula-Stein kernel is $k_{C_0}(u,v)=\langle\xi_{C_0}(u),\xi_{C_0}(v)\rangle_{\mathcal H^d}$ and the empirical discrepancy is
\begin{align}
\widehat{\mathrm{CSD}}_n^2=\frac1{n^2}\sum_{i=1}^n\sum_{j=1}^n k_{C_0}(U_i,U_j). 
\end{align}

Unless noted, we work with test functions generated by a boundary-conditioned RKHS (see Proposition 4.2 below), so that each component has zero trace on the coordinate faces. This eliminates boundary terms in the Stein identity
on $(0,1)^d$.

\begin{definition}[Copula Stein discrepancy]
\label{def:CSD}
Given a target copula $C_0$ and a candidate distribution $Q$ on $(0,1)^d$,
\begin{align}
\mathrm{CSD}(Q,C_0)\;\equiv\; \sup_{\|g\|_{\mathcal H^d}\le 1} \,\mathbb E_{U\sim Q}\big[\,\TC g(U)\,\big].
\end{align}
Under Assumption~\ref{assum:unified} (A1–A4), we have $\mathrm{CSD}(C_0,C_0)=0$ by Lemma~\ref{lem:stein_identity}.
\end{definition}

\textbf{Standing Assumptions (used throughout).}
\begin{assumption}[Regularity Conditions]\label{assum:unified}
The following hold for the target copula $C_0$ and base kernel $k$:
\begin{enumerate}[label=\textbf{A\arabic*:}]
    \item \textbf{(Copula Smoothness)} $C_0$ has a density $c_0$ positive and $C^1$ on $(0,1)^d$. For Archimedean $C_0$, the generator $\varphi_0$ is $C^{d+1}$ on $(0,1)$.
    \item \textbf{(Score Growth)} The score $s_0(u)=\nabla\log c_0(u)$ exists on $(0,1)^d$ and is polynomially bounded near the boundary: $\|s_0(u)\|\le M_s\!\Big(1+\sum_{j=1}^d\min(u_j,1-u_j)^{-\beta}\Big)$ for some $M_s>0$, $\beta\in[0,1)$.
    \item\textbf{(Stein-kernel universality)} The Copula--Stein kernel $k_{C_0}$ defined in Proposition 5.2 admits a bounded, continuous extension to the compact domain $[0,1]^d$ and is $c$-universal on $[0,1]^d$ (equivalently, its RKHS is dense in $C([0,1]^d)$ under the sup norm).

\item\textbf{(Boundary control, concrete sufficient condition)} The Stein identity for $\mathcal{T}_{C_0}$ holds without boundary terms. A sufficient condition (and the one we adopt for all results below) is obtained by working with a boundary-conditioned kernel: choose a base positive definite kernel $\bar k$ on $[0,1]^d$ with $\bar k \in C^2$, define
\begin{align}
&b(u) := \prod_{j=1}^d u_j(1-u_j), \\ &k(u,v) := b(u)\,b(v)\,\bar k(u,v),
\end{align}
and let $\mathcal{H}$ be the RKHS of $k$. Then every $f \in \mathcal{H}$ vanishes on $\partial[0,1]^d$ and hence every $g \in \mathcal{H}^d$ has zero trace on each coordinate face (see Proposition 4.2). In particular, integration by parts yields the Stein identity.
    \item \textbf{(Moment Condition)} $\E_{U\sim C_0}[k_{C_0}(U,U)^2]<\infty$.
\end{enumerate}
\end{assumption}

\begin{proposition}[Boundary-conditioned kernels yield the required trace condition]\label{prop:boundary_conditioned_kernel}
Let $\bar k$ be a positive definite kernel on $[0,1]^d$ with RKHS $\bar{\mathcal{H}}$. Let $b:[0,1]^d\to\mathbb{R}$ be continuous and define
\begin{align}
k(u,v) := b(u)\,b(v)\,\bar k(u,v), \qquad u,v\in[0,1]^d.
\end{align}
Then $k$ is positive definite. Moreover, the mapping
\begin{align}
M_b:\bar{\mathcal{H}} \to \mathcal{F}([0,1]^d), \qquad (M_b \bar f)(u) := b(u)\bar f(u),
\end{align}
has range contained in the RKHS $\mathcal{H}$ of $k$, and every $f\in \mathcal{H}$ admits a representative that vanishes wherever $b$ vanishes. In particular, if $b(u)=0$ for all $u\in\partial[0,1]^d$ (e.g., $b(u)=\prod_{j=1}^d u_j(1-u_j)$), then every $f\in \mathcal{H}$ has zero trace on $\partial[0,1]^d$, hence every $g\in \mathcal{H}^d$ has zero trace on each coordinate face.

If additionally $\bar k\in C^2([0,1]^d\times[0,1]^d)$, then $k\in C^2([0,1]^d\times[0,1]^d)$ and the reproducing identities for first-order partial derivatives used in Proposition 5.1 remain valid (on compact subsets of $(0,1)^d$).
\end{proposition}

\begin{remark}[Random-feature / Nystr\"om scaling becomes explicit]\label{rem:rf}
If $\bar k$ admits a feature representation $\bar k(u,v)\approx \phi(u)^\top \phi(v)$, then the boundary-conditioned kernel satisfies
\begin{align}
k(u,v)\approx (b(u)\phi(u))^\top (b(v)\phi(v)).
\end{align}
Thus any random-feature or Nystr\"om approximation built for $\bar k$ immediately yields an approximation for $k$ by multiplying features by $b(\cdot)$. This makes the near-linear-time approximation claim constructive once $\bar k$ is chosen.
\end{remark}

\textbf{Strengthening used only for metrization.}
When needed for metrization, we additionally assume:

\begin{assumption}[Boundary rates for metrization]\label{assum:metrization_rates}
There exists $\gamma\in(0,1)$ such that $g,\nabla g=o\!\big(\prod_{j=1}^d \min(u_j,1-u_j)^{\gamma}\big)$ for all $g\in\mathcal H^d$ as $u$ approaches the boundary, and $\|s_0(u)\|=O\!\big(\sum_{j=1}^d \min(u_j,1-u_j)^{-\gamma}\big)$.
\end{assumption}

\textbf{Estimator (fixed):}
We use the V-statistic form throughout:
\begin{align}
\widehat{\mathrm{CSD}}_n^2 \;=\; \frac{1}{n^2}\sum_{i=1}^n\sum_{j=1}^n k_{C_0}(U_i,U_j),
\quad U_i \in [0,1]^d.
\end{align}

\textbf{Archimedean score sum:}
Let $\varphi:[0,1]\to\mathbb{R}$ be the generator\slash score map; we write
\begin{align}
t(u) \;=\; \sum_{k=1}^d \varphi(u_k)\,.
\end{align}

\textbf{Random features (Section~\ref{sec:computation}):}
We denote by $m\in\mathbb{N}$ the number of random features and by
$\phi:\,[0,1]^d \to \mathbb{R}^m$ the random feature map used to approximate $k_{C_0}$.
Given a draw $\omega\sim\Pi$ (feature distribution), write 
$\phi(u)=\phi(u;\omega_{1:m})$ for the $m$-dimensional feature vector.

We use a random-feature witness defined in Section~\ref{sec:computation}.

\section{The Copula-Stein Discrepancy}
\label{sec:csd}

The central goal of this section is to construct a dependence-aware goodness-of-fit metric relative to a target copula $C_0$. Throughout, we work under Assumptions~\ref{assum:unified} (and, when invoked for metrization results, Assumption~\ref{assum:metrization_rates}).

\subsection{From IPM to a Computable Kernel}

Following the framework of integral probability metrics (IPMs) \citep{Muller1997IPM, Zolotarev1983}, we define a discrepancy as the worst-case difference in expectation over a class of test functions $g \in \mathcal{H}^d$:
\begin{align}
&\text{CSD}(C, C_0) \\ 
&= \sup_{\|g\|_{\mathcal{H}^d} \leq 1} \left|\E_{U \sim C} [\TC g(U)] - \E_{U \sim C_0} [\TC g(U)]\right|.
\label{eq:csd_ipm_def}
\end{align}
The key to making this definition practical is the construction of an operator $\TC$ that satisfies a crucial mean-zero property with respect to the target distribution $C_0$. We design our operator based on the principles of Stein's method \citep{Stein1986Approx, LiuLeeJordan2016}.

\begin{definition}[Copula-Stein Operator]
For a target copula $C_0$ with density $c_0$, define
\begin{align}
    \TC g(u) = \sum_{j=1}^d \left[ \partial_{u_j}g_j(u) + g_j(u)\,\partial_{u_j}\log c_0(u) \right].
\end{align}
\end{definition}

\begin{lemma}[Stein identity and boundary conditions]
\label{lem:stein_identity}
Under \textbf{A1}–\textbf{A2}, for any $g\in \mathcal H^d$ with each $g_j$ vanishing on the faces $\{u_j=0\}$ and $\{u_j=1\}$,
\begin{align}
\mathbb E_{U\sim C_0}\!\big[\TC g(U)\big]=0,\quad
\TC g := \sum_{j=1}^d \big( \partial_{u_j} g_j + s_j g_j\big),\\ s_j=\partial_{u_j}\log c_0 .
\end{align}
\end{lemma}

Proof is given in Appendix~\ref{app:proof1}.

Using Lemma~\ref{lem:stein_identity}, the zero-mean property immediately simplifies \eqref{eq:csd_ipm_def}. For an empirical measure $C_n = \frac{1}{n}\sum_{i=1}^n \delta_{U_i}$ from a sample, the second expectation vanishes, leaving
\begin{align}
    \text{CSD}(C_n, C_0) = \sup_{\|g\|_{\mathcal{H}^d} \le 1} \left| \frac{1}{n} \sum_{i=1}^n \TC g(U_i) \right|.
    \label{eq:csd_simplified_def}
\end{align}
While simpler, the supremum over the infinite-dimensional unit ball in $\mathcal{H}^d$ remains intractable. Our second key tool is the reproducing property of the RKHS \citep{BerlinThomas2004RKHS}, which allows us to represent the operator's action as an inner product with a single “witness function.”

\begin{proposition}[Witness Function Representation]
\label{prop:witness_detailed}
For any $u \in (0,1)^d$, there exists a witness function $\xi_{C_0}(u) \in \mathcal{H}^d$ such that
\begin{align}
    \TC g(u) = \langle g, \xi_{C_0}(u) \rangle_{\mathcal{H}^d}
\end{align}
for all $g \in \mathcal{H}^d$.
\end{proposition}

Proof is given in Appendix~\ref{app:proof2}.

This witness representation allows us to solve the supremum in \eqref{eq:csd_simplified_def} analytically. The absolute value is redundant for $C_n$ since $\E_{C_0}[\TC g]=0$ by Lemma~\ref{lem:stein_identity}. By Hilbert space duality, the supremum is attained when $g$ aligns with the mean witness, converting the expression into a norm; squaring yields a closed-form V-statistic with the Copula-Stein kernel.

\begin{theorem}[Closed-Form CSD]
\label{thm:closed_form_detailed}
The squared Copula-Stein Discrepancy for the empirical copula $C_n=\frac{1}{n}\sum_{i=1}^n \delta_{U_i}$ is
\begin{align}
    \text{CSD}^2(C_n, C_0) = \frac{1}{n^2} \sum_{i=1}^n \sum_{j=1}^n k_{C_0}(U_i, U_j),
\end{align}
where $k_{C_0}(u,v) = \langle \xi_{C_0}(u), \xi_{C_0}(v) \rangle_{\mathcal{H}^d}$ is the Copula-Stein kernel.
\end{theorem}

We have given the proof in Appendix~\ref{app:proof3}.

The kernel $k_{C_0}(u,v)$ is the computational heart of our method. We provide its explicit formula below.

\subsection{A practical test}
In practice, we estimate the discrepancy by the V-statistic form
\begin{align}
\widehat{\mathrm{CSD}}_n^2 \;=\; \frac{1}{n^2}\sum_{i=1}^n\sum_{j=1}^n k_{C_0}(U_i,U_j),
\end{align}
which aligns with our asymptotic theory (Sec.~\ref{sec:theory}) and our implementations (Sec.~\ref{sec:computation}). We obtain critical values via a parametric bootstrap under the null copula $C_0$:
for $b=1,\dots,B$, draw an i.i.d.\ bootstrap sample $U^{\ast(b)}_{1:n} \sim C_0$ (with the null parameter fixed at $\theta_0$, or using a plug-in estimate $\hat\theta$ when applicable), and compute
\begin{align}
T^{\ast(b)}=\frac1{n^2}\sum_{i=1}^n\sum_{j=1}^n k_{C_0}\!\bigl(U^{\ast(b)}_i,U^{\ast(b)}_j\bigr).
\end{align}
Repeat $B$ times (e.g., $B=1000$) and reject for
$\widehat{\mathrm{CSD}}_n^2 > q_{1-\alpha}\{T^\ast\}$, where $q_{1-\alpha}\{T^\ast\}$ denotes the empirical $(1-\alpha)$-quantile of $\{T^{\ast(b)}\}_{b=1}^B$.

\begin{remark}
The unbiased U-statistic $\frac{1}{n(n-1)}\sum_{i\ne j}k_{C_0}(U_i,U_j)$ is a common variant; we adopt the V-statistic for consistency with theory and streaming/GPU algorithms.
\end{remark}

\begin{proposition}[Explicit Kernel Formula]
\label{prop:kernel_formula_detailed}
Let $s(u) = \nabla \log c_0(u)$ denote the copula score function. The Copula-Stein kernel admits the explicit form:
\begin{align}
    k_{C_0}(u,v) &= s(u)^T s(v) k(u,v) \\ &+s(u)^T \nabla_v k(u,v) + s(v)^T \nabla_u k(u,v) \\ &+ \text{tr}(\nabla_u \nabla_v^T k(u,v)).
\end{align}
\end{proposition}

Full proof is provided in Appendix~\ref{app:proof4}.

\subsection{Application to Archimedean Copulas}

To make the kernel practical, we need an analytic form for the score function $s(u)$. For the broad and flexible class of Archimedean copulas, this score can be derived directly from the generator $\varphi_0$ \citep{McNeilNeslehova2009}.

\begin{proposition}[Archimedean Score Function]
\label{prop:archimedean_score}
Let $C$ be an Archimedean copula on $(0,1)^d$ with generator $\varphi:(0,1]\to[0,\infty)$ strictly decreasing and twice differentiable (Assumptions~\ref{assum:unified}), and let $\psi=\varphi^{-1}$. For $u\in(0,1)^d$, set
\begin{align}
t \;=\; \sum_{k=1}^d \varphi(u_k).
\end{align}
Then the coordinate-wise score of the copula density $c$ is
\begin{align}
\;\frac{\partial}{\partial u_j}\log c(u)
\;=\; \varphi'(u_j)\,\frac{\psi^{(d+1)}(t)}{\psi^{(d)}(t)}
\;+\; \frac{\varphi''(u_j)}{\varphi'(u_j)}\;,\\ j=1,\dots,d.
\end{align}
\end{proposition}

Proof is provided in Appendix~\ref{app:proof5}.

We used $\frac{d}{du}\log(\varphi'(u))=\varphi''(u)/\varphi'(u)$ since $\varphi'(u)<0$ for Archimedean generators.

\begin{remark}
Let $w=\psi(t)$. Then $\partial_t=\frac{1}{\varphi'(w)}\partial_w$, so one can equivalently write
\begin{align}
\frac{\partial}{\partial u_j}\log c(u)
&=\Big[\varphi'(w)\,\partial_t\log\psi^{(d)}(t)\Big]\frac{\varphi'(u_j)}{\varphi'(w)}
+\frac{\varphi''(u_j)}{\varphi'(u_j)}\\
&=\varphi'(u_j)\,\partial_t\log\psi^{(d)}(t)+\frac{\varphi''(u_j)}{\varphi'(u_j)}.
\end{align}
One should avoid replacing $\partial_t\log\psi^{(d)}(t)$ by ad-hoc identities in terms of $\varphi^{(\cdot)}(w)$ unless one explicitly invokes the inverse-derivative 
\end{remark}


\section{Theoretical Properties}
\label{sec:theory}

Having established the Copula-Stein Discrepancy, we now develop the theoretical foundations that justify its use as a principled statistical tool. Our analysis answers a sequence of fundamental questions: (i) Is CSD a valid metric that uniquely identifies the target copula? (ii) Does its empirical estimator converge at an optimal rate? (iii) Does it possess the claimed sensitivity to tail dependence? and (iv) Is the convergence rate fundamentally optimal? We answer each of these affirmatively in the following theorems.

\subsection{Metrization and Consistency}

The foundational property of any discrepancy measure is that it vanishes if and only if the compared distributions are identical. For CSD, this equivalence extends to characterizing weak convergence of probability measures on the unit hypercube.

\begin{theorem}[Detectability]
\label{thm:detectability}
Under Assumptions \ref{assum:unified} (A1–A4),  
$\mathrm{CSD}(Q,C_0)=0$ if and only if $Q=C_0$.
\end{theorem}

Proof is given in Appendix~\ref{app:proof6}.

\begin{theorem}[Metrization of weak convergence]\label{thm:metrization}
Under Assumption A3, the copula-Stein discrepancy $\mathrm{CSD}(\cdot,C_0)$ metrizes weak convergence:
for any sequence $(Q_n)$ in $\mathcal P([0,1]^d)$ and any $Q\in\mathcal P([0,1]^d)$,
\begin{align}
\mathrm{CSD}(Q_n,C_0)\to \mathrm{CSD}(Q,C_0)\quad\Longleftrightarrow\quad Q_n \Rightarrow Q .
\end{align}
\end{theorem}

Proof is provided in Appendix~\ref{app:proof7}.

\subsection{Finite‐Sample Convergence and Optimality (Detailed Proof)}

\begin{theorem}[Finite‐Sample Convergence Rate]
\label{thm:finite_sample_detailed}
Let $U_1,\dots,U_n\iid C_0$, and define
\begin{align}
\widehat{\mathrm{CSD}}_n^2 \;=\;\frac1{n^2}\sum_{i=1}^n\sum_{j=1}^n k_{C_0}(U_i,U_j).
\end{align}
Under Assumptions \ref{assum:unified}, one has
\begin{align}
\E\bigl[\widehat{\mathrm{CSD}}_n^2\bigr]
=\frac1n\,\E_{U\sim C_0}\bigl[k_{C_0}(U,U)\bigr],
\end{align}
and hence $\E[\widehat{\mathrm{CSD}}_n^2]=O(n^{-1})$.
\end{theorem}

Full proof is given in Appendix~\ref{app:proof8}.

\begin{corollary}[Optimal High‐Probability Convergence Rate]
\label{cor:optimal_rate_detailed}
Under the unified assumptions, the CSD estimator satisfies
\begin{align}
\widehat{\mathrm{CSD}}_n = O_P\bigl(n^{-1/2}\bigr).
\end{align}
\end{corollary}

Proof is provided in Appendix~\ref{app:proof9}.

The preceding results establish that CSD is a well-behaved and statistically optimal estimator. However, its primary motivation is to overcome the blindness of traditional discrepancies to important dependence structures. We now prove that CSD is, by construction, sensitive to the crucial property of tail dependence.

\subsection{Sensitivity to Tail Dependence}

The distinguishing feature of CSD lies in its ability to detect differences in copula tail behavior that are invisible to conventional discrepancies. We formalize this intuition through our main theoretical contribution.

\begin{theorem}[Tail-dependence sensitivity]
\label{thm:tail-sensitivity}
Assume \textbf{A1}–\textbf{A4}. Fix $d=2$. Let $C_1$ and $C_2$ be copulas with well defined lower tail coefficients $\lambda_L(C_1)$ and $\lambda_L(C_2)$. Let $\mathrm{CSD}_{C_2}(C_1)\equiv \mathrm{CSD}(C_1,C_2)$ denote the copula Stein discrepancy that uses $C_2$ as target. If $\lambda_L(C_1)\neq \lambda_L(C_2)$ then
\begin{align}
\mathrm{CSD}_{C_2}(C_1) \;>\; 0.
\end{align}
The same conclusion holds for the upper tail with $\lambda_U$ and a witness localized near the upper corner.
\end{theorem}

Detailed proof is given in Appendix~\ref{app:proof10}.

Finally, to complete our theoretical analysis, we prove that the $O_P(n^{-1/2})$ convergence rate is not merely an artifact of our estimator, but is the fundamental statistical limit for this problem \citep{Tsybakov2009Nonparam}.

Having established the $O(n^{-1})$ convergence rate, we now prove this rate is minimax optimal, confirming that no estimator can achieve faster convergence.

\begin{theorem}[Minimax risk for the squared discrepancy]
\label{thm:minimax}
Let $\widehat{\mathrm{CSD}}_n^2$ be any estimator of $D^2(Q)\equiv \mathrm{CSD}(Q,C_0)^2$ based on an i.i.d. sample of size $n$ from $Q$. Work over the local total-variation neighborhood
\begin{align}
\mathcal Q \equiv \Big\{Q: d_{\mathrm{TV}}(Q,C_0)\le \kappa n^{-1/2}\Big\}.
\end{align} 
(For the lower bound we do not restrict $\mathcal Q$ to copulas; the two-point construction below produces nearby alternatives dominated by $C_0$.) Under \textbf{A1}–\textbf{A3} there exist constants $0<c\le C<\infty$ such that
\begin{align}
c\,n^{-2}\ \le\ \inf_{\widehat{\mathrm{CSD}}_n^2}\ \sup_{Q\in\mathcal Q}\ \mathbb E_Q\big[\big(\widehat{\mathrm{CSD}}_n^2 - D^2(Q)\big)^2\big]
\ \le\ C\,n^{-2}.
\end{align}
Hence the minimax mean squared error for the squared discrepancy is of order $n^{-2}$.
\end{theorem}

Detailed proof is provided in Appendix~\ref{app:proof11}.

\subsection{Summary of Theoretical Guarantees}

The theorems in this section provide a complete theoretical justification for CSD as a tool for copula analysis. We have established that CSD:
\begin{enumerate}
\item Is a valid metric, as it metrizes weak convergence on the space of copula distributions (Theorem \ref{thm:metrization}).
\item Is statistically optimal for finite samples, with its empirical estimator converging at the minimax optimal rate of $O_P(n^{-1/2})$ (Theorem \ref{thm:finite_sample_detailed}, Corollary \ref{cor:optimal_rate_detailed}, and Theorem \ref{thm:minimax}).
\item Is sensitive to fine-grained dependence, provably distinguishing between copulas with different tail dependence coefficients (Theorem \ref{thm:tail-sensitivity}).
\end{enumerate}
Together, these results confirm that CSD is a principled, practical, and powerful tool for dependence modeling.

\begin{remark}[Extension]
Although our exposition focuses on Archimedean copulas for concreteness, the CSD framework relies only on the unified regularity conditions of the copula score $s(u)=\nabla \log c(u)$, and therefore extends to arbitrary copulas with sufficiently regular densities. We provide the general Copula-Stein operator and kernel formula, and show how the main guarantees (zero-mean, metrization, and finite-sample rate) carry over, along with examples for Gaussian, vine, and mixture copulas, in Appendix~\ref{sec:beyond_archimedean}.
\end{remark}

\begin{remark}[Computational Complexity ]
CSD is computationally tractable: with a separable $C^2$ base kernel, each Copula--Stein kernel evaluation costs $O(d)$, giving an exact estimator computable in $O(n^2 d)$ time, while a random-feature approximation reduces the cost to $O(nmd)$ (nearly linear in $n$ when $m\ll n$). Full derivations and implementation details, including the exact-kernel routine, random-feature algorithm, and GPU and streaming variants, are provided in Appendix~\ref{sec:computation}.
\end{remark}

\section{Numerical Experiments}
\label{sec:experiments}
\subsection{Experiment 1: empirical level and power (bivariate)}
We study the finite-sample behavior of the CSD goodness-of-fit test in dimension $d=2$.
Given i.i.d.\ pseudo-observations $U_1,\dots,U_n\in(0,1)^2$, we compute the V-statistic
\begin{align}
\widehat{\mathrm{CSD}}_n^2 \;=\; \frac{1}{n^2}\sum_{i=1}^n\sum_{j=1}^n k_{C_0}(U_i,U_j),
\end{align}
where $k_{C_0}$ is the Copula-Stein kernel defined in Prop.~5.2. 
We use the boundary-conditioned product RBF kernel with $b(u)=\prod_{j=1}^2 u_j(1-u_j)$ and choose the bandwidth $h$ by the median heuristic.

We test the null model $C_0$ given by a bivariate Gumbel copula with Kendall's $\tau=0.5$ (equivalently, $\theta_g=2$).
Critical values are obtained by a parametric bootstrap under the null:
for each Monte Carlo replicate we generate $B=1000$ bootstrap samples of size $n$ from the Gumbel$(\theta_g)$ model, recompute $\widehat{\mathrm{CSD}}_n^2$, and reject at level $\alpha=0.05$ using the empirical $(1-\alpha)$ quantile.
We report the empirical Type I error under the null and power under two alternatives:
(i) a Clayton copula with matched Kendall's $\tau=0.5$ (so $\theta_c=2$), and
(ii) a misspecified Gumbel copula with $\theta_{\text{alt}}=2.4$.
All rejection rates are estimated over $R=500$ Monte Carlo repetitions.

\begin{table}[t]
\centering
\caption{Experiment 1 (bivariate): empirical Type I error and power for the CSD test ($\alpha=0.05$, $R=500$, $B=1000$).}
\label{tab:exp1}
\begin{tabular}{rccc}
\toprule
$n$
& \makecell{Type I\\(Gumbel \\$\theta_g=2$)}
& \makecell{Power\\(Clayton \\$\theta_c=2$)}
& \makecell{Power\\(Gumbel \\$\theta_{\text{alt}}=2.4$)} \\
\midrule
200  & 0.046 & 0.104 & 0.748 \\
500  & 0.064 & 0.268 & 1.000 \\
1000 & 0.064 & 0.426 & 1.000 \\
\bottomrule
\end{tabular}
\end{table}

Overall, from Table~\ref{tab:exp1}, we see that the test exhibits near-nominal Type I error at $\alpha=0.05$ across sample sizes. We observe a slight liberal tendency at larger $n$ (Type I $\approx 0.064$ at $n\in\{500,1000\}$), which we attribute to finite-sample and Monte Carlo variability from bootstrap calibration and numerical score evaluation.

Power increases with $n$ for both alternatives: the test strongly detects within-family misspecification of the Gumbel parameter (power $0.748$ at $n=200$ and $1.0$ by $n=500$), while power against the Clayton alternative with matched Kendall’s $\tau$ is more moderate, reflecting a more subtle departure from the null at fixed $\tau$.

\subsection{Experiment 2: random-feature approximation accuracy for $\mathrm{CSD}^2$}
We evaluate the accuracy of the random-feature approximation used for scalable CSD.
We generate a bivariate sample of size $n=5000$ from a Gumbel copula with Kendall's $\tau=0.5$ (so $\theta=2$), compute the \emph{exact} $\widehat{\mathrm{CSD}}_n^2$ using the Copula-Stein kernel (Prop.~\ref{prop:kernel_formula_detailed}), and compare it to the random-feature approximation based on Algorithm~\ref{alg:random_features_corrected}.
The bandwidth $h$ is set by the median heuristic.
For each number of random features $m$, we repeat the random-feature construction $50$ times (independent draws of random features) and report the mean and standard deviation of the approximation and its relative error
$|\mathrm{approx}-\mathrm{exact}|/|\mathrm{exact}|$.

\begin{table}[t]
\centering
\caption{Experiment 2: random-feature approximation accuracy for $\mathrm{CSD}^2$ ($n=5000$, $\tau=0.5$, 50 independent random-feature draws per $m$). Exact $\widehat{\mathrm{CSD}}_n^2 = 5.893\times 10^{-4}$.}
\label{tab:exp2}
\small
\setlength{\tabcolsep}{2pt}
\renewcommand{\arraystretch}{1.05}
\begin{tabular}{rcccc}
\toprule
$m$ & Approx.\ mean & Approx.\ s.d. & Rel.\ err.\ mean & Rel.\ err.\ s.d. \\
\midrule
250   & $5.928\times 10^{-4}$ & $2.60\times 10^{-5}$ & $3.53\times 10^{-2}$ & $2.66\times 10^{-2}$ \\
500   & $5.835\times 10^{-4}$ & $2.12\times 10^{-5}$ & $2.94\times 10^{-2}$ & $2.26\times 10^{-2}$ \\
1000  & $5.890\times 10^{-4}$ & $1.51\times 10^{-5}$ & $2.05\times 10^{-2}$ & $1.50\times 10^{-2}$ \\
2000  & $5.894\times 10^{-4}$ & $1.02\times 10^{-5}$ & $1.34\times 10^{-2}$ & $1.08\times 10^{-2}$ \\
5000  & $5.887\times 10^{-4}$ & $6.72\times 10^{-6}$ & $9.76\times 10^{-3}$ & $5.83\times 10^{-3}$ \\
10000 & $5.899\times 10^{-4}$ & $4.23\times 10^{-6}$ & $5.90\times 10^{-3}$ & $4.12\times 10^{-3}$ \\
\bottomrule
\end{tabular}
\end{table}

\begin{figure}[t]
\centering
\includegraphics[width=0.72\linewidth]{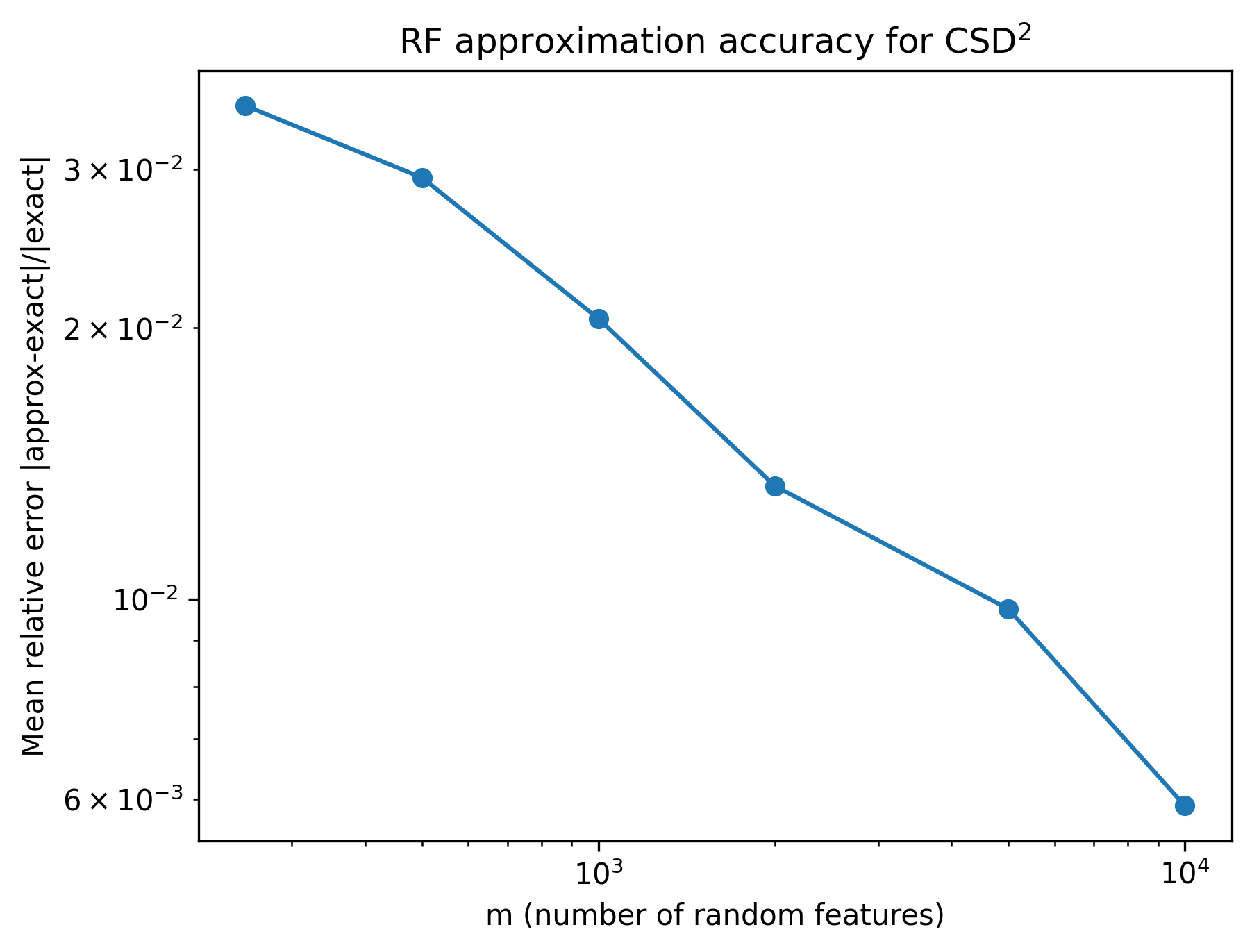}
\caption{Experiment 2: mean relative error of the random-feature approximation versus the exact $\widehat{\mathrm{CSD}}_n^2$, as a function of the number of random features $m$ (log-log scale).}
\label{fig:exp2}
\end{figure}

From Table~\ref{tab:exp2} and Figure~\ref{fig:exp2}, we see that as $m$ increases, the approximation rapidly concentrates around the exact value: the mean relative error decreases from about $3.5\%$ at $m=250$ to below $1\%$ by $m=5000$, and to approximately $0.6\%$ at $m=10000$, supporting the accuracy of the scalable random-feature implementation.


\section{Conclusion and Future Work}

In this work, we introduced the Copula-Stein Discrepancy (CSD), a dependence-structure-aware framework for copula goodness-of-fit testing built by constructing a Stein operator directly on copula densities. Our theory provides a comprehensive treatment of Stein methods in the copula domain: CSD metrizes weak convergence on copula distributions, achieves the minimax-optimal convergence rate of $O_P(n^{-1/2})$, and is provably sensitive to tail-dependence differences via explicit lower bounds. For Archimedean copulas, we derived closed-form expressions for the CSD kernel, and our framework applies more broadly to copulas satisfying mild regularity conditions (including elliptical, vine, and mixture constructions).

We also addressed computational bottlenecks via exact computation with linear scaling in dimension and random-feature approximations that reduce the $O(n^2)$ cost to nearly linear time. Our numerical results support these claims: Experiment~1 shows near-nominal Type I error at $\alpha=0.05$ with increasing power as $n$ grows, and Experiment~2 demonstrates that the random-feature approximation rapidly concentrates around the exact $\widehat{\mathrm{CSD}}_n^2$ as the number of features increases.

The framework developed here suggests several directions for further research. Immediate extensions include (i) a CSD-based two-sample test for comparing dependence structures, (ii) using CSD as an objective for copula parameter estimation as a robust alternative to maximum likelihood, and (iii) discrepancy-based model selection across copula families. On the methodological side, promising directions include incorporating CSD as a regularizer in deep generative models (e.g., VAEs/GANs) to improve dependence fidelity, extending the theory to high-dimensional regimes where $d$ grows with $n$, and developing streaming/online variants for real-time dependence monitoring.

Finally, broader empirical validation remains important, including benchmark comparisons against established copula GOF tests, real-data studies in domains where tail dependence is central (e.g., finance, hydrology, neuroscience), and open-source implementations in R and Python to support adoption and reproducibility.

\section*{Impact Statement}

CSD is a statistical tool for assessing the goodness-of-fit of copula models, with an explicit focus on whether a proposed model captures multivariate dependence, including tail dependence, rather than on marginal behavior. If adopted, it could help researchers and practitioners more reliably validate dependence models in areas where joint extremes matter (for example, finance or hydrology), and it includes computational strategies intended to make such validation practical at larger sample sizes. However, like any model-checking method, CSD can be misused if its results are treated as a substitute for domain judgment or if it is deployed in high-stakes settings without careful attention to data quality, uncertainty, and downstream consequences. In particular, a well-fitting dependence model does not imply a system is fair or appropriate for decision-making, and failures can arise from distribution shift or incorrect preprocessing rather than the copula family alone. We view CSD as a diagnostic component within a broader responsible modeling workflow, where conclusions are supported by sensitivity analyses, transparent reporting, and domain-informed validation.
\bibliography{references}
\bibliographystyle{icml2026}


\newpage
\appendix
\onecolumn

\section{Proofs}
\label{app:proofs}

In this section, we provide the proofs of the lemmas, propositions, theorems and corollaries stated in the main text.

\subsection{Proof of Lemma~\ref{lem:stein_identity}}
\label{app:proof1}

\begin{proof}
For each $j$,
\begin{align}
\int \partial_{u_j} g_j(u)\, c_0(u)\,du
= \Big[g_j(u)\,c_0(u)\Big]_{u_j=0}^{u_j=1} - \int g_j(u)\,\partial_{u_j} c_0(u)\,du.
\end{align}
Boundary vanishing implies the bracketed term is $0$, and $\partial_{u_j} c_0 = s_j c_0$. Summing over $j$ yields $\int \TC g \, c_0\,du = 0$, i.e. $\mathbb E_{C_0}[\TC g(U)]=0$.
\end{proof}

\subsection{Proof of Proposition~\ref{prop:witness_detailed}}
\label{app:proof2}

\begin{proof}
Let $\mathcal H$ be the scalar RKHS with kernel $k$, so $\mathcal H^d$ is the $d$-fold product with
$\langle g,h\rangle_{\mathcal H^d}=\sum_{j=1}^d \langle g_j,h_j\rangle_{\mathcal H}$.
By reproducing and closure under partials, $\partial_{u_j} g_j(u)=\langle g_j,\, \partial_{u_j} k(u,\cdot)\rangle_{\mathcal H}$ and $g_j(u)=\langle g_j,\, k(u,\cdot)\rangle_{\mathcal H}$. Hence
\begin{align}
\TC g(u)
= \sum_{j=1}^d \Big\langle g_j,\; \partial_{u_j} k(u,\cdot) + (\partial_{u_j}\log c_0(u))\,k(u,\cdot)\Big\rangle_{\mathcal H}
= \big\langle g,\, \xi_{C_0}(u)\big\rangle_{\mathcal H^d},
\end{align}
with component $\xi_{C_0}(u)_j := \partial_{u_j} k(u,\cdot) + s_j(u)\,k(u,\cdot)\in\mathcal H$. Thus $\xi_{C_0}(u)\in\mathcal H^d$.
\end{proof}

\subsection{Proof of Theorem~\ref{thm:closed_form_detailed}}
\label{app:proof3}

\begin{proof}
By Lemma~\ref{lem:stein_identity},
\begin{align}
\text{CSD}(C_n,C_0)
= \sup_{\|g\|_{\mathcal H^d}\le 1}
\left|\frac{1}{n}\sum_{i=1}^n \TC g(U_i)\right|
= \sup_{\|g\|\le 1} \left|\left\langle g,\; \mu_n\right\rangle_{\mathcal H^d}\right|,
\quad
\mu_n:=\frac{1}{n}\sum_{i=1}^n \xi_{C_0}(U_i).
\end{align}
By Cauchy–Schwarz, the supremum equals $\|\mu_n\|_{\mathcal H^d}$ (attained at $g=\mu_n/\|\mu_n\|$ when $\mu_n\neq 0$). Then
\begin{align}
\text{CSD}^2(C_n,C_0)=\|\mu_n\|_{\mathcal H^d}^2
= \left\langle \frac{1}{n}\sum_i \xi_{C_0}(U_i),\; \frac{1}{n}\sum_j \xi_{C_0}(U_j)\right\rangle_{\mathcal H^d}
= \frac{1}{n^2}\sum_{i=1}^n\sum_{j=1}^n \langle \xi_{C_0}(U_i), \xi_{C_0}(U_j)\rangle_{\mathcal H^d},
\end{align}
which identifies $k_{C_0}(u,v):=\langle \xi_{C_0}(u),\xi_{C_0}(v)\rangle_{\mathcal H^d}$.
\end{proof}

\subsection{Proof of Proposition~\ref{prop:kernel_formula_detailed}}
\label{app:proof4}

\begin{proof}
From the witness function definition in Proposition~\ref{prop:witness_detailed}:
\begin{align}
    \xi_{C_0}^{(j)}(u) = \frac{\partial k(\cdot, u)}{\partial u_j} + s_j(u) k(\cdot, u),
\end{align}
where $s_j(u) = \frac{\partial \log c_0(u)}{\partial u_j}$.

The Copula-Stein kernel is:
\begin{align}
    k_{C_0}(u,v) &= \langle \xi_{C_0}(u), \xi_{C_0}(v) \rangle_{\mathcal{H}^d} \\
    &= \sum_{j=1}^d \langle \xi_{C_0}^{(j)}(u), \xi_{C_0}^{(j)}(v) \rangle_{\mathcal{H}} \\
    &= \sum_{j=1}^d \left\langle \frac{\partial k(\cdot, u)}{\partial u_j} + s_j(u) k(\cdot, u), \frac{\partial k(\cdot, v)}{\partial v_j} + s_j(v) k(\cdot, v) \right\rangle_{\mathcal{H}}.
\end{align}

Expanding each inner product using bilinearity:
\begin{align}
    &= \sum_{j=1}^d \left[ \left\langle \frac{\partial k(\cdot, u)}{\partial u_j}, \frac{\partial k(\cdot, v)}{\partial v_j} \right\rangle_{\mathcal{H}} + s_j(u) \left\langle k(\cdot, u), \frac{\partial k(\cdot, v)}{\partial v_j} \right\rangle_{\mathcal{H}} \right. \\
    &\quad + s_j(v) \left\langle \frac{\partial k(\cdot, u)}{\partial u_j}, k(\cdot, v) \right\rangle_{\mathcal{H}} + s_j(u) s_j(v) \langle k(\cdot, u), k(\cdot, v) \rangle_{\mathcal{H}} \bigg].
\end{align}

Now we apply the reproducing property of the RKHS. For any $f, g \in \mathcal{H}$ and points $a, b$:
\begin{align}
    \langle k(\cdot, a), k(\cdot, b) \rangle_{\mathcal{H}} &= k(a, b), \\
    \left\langle \frac{\partial k(\cdot, a)}{\partial a_j}, k(\cdot, b) \right\rangle_{\mathcal{H}} &= \frac{\partial k(a, b)}{\partial a_j}, \\
    \left\langle k(\cdot, a), \frac{\partial k(\cdot, b)}{\partial b_j} \right\rangle_{\mathcal{H}} &= \frac{\partial k(a, b)}{\partial b_j}, \\
    \left\langle \frac{\partial k(\cdot, a)}{\partial a_j}, \frac{\partial k(\cdot, b)}{\partial b_j} \right\rangle_{\mathcal{H}} &= \frac{\partial^2 k(a, b)}{\partial a_j \partial b_j}.
\end{align}

Applying these identities:
\begin{align}
    k_{C_0}(u,v) &= \sum_{j=1}^d \left[ \frac{\partial^2 k(u, v)}{\partial u_j \partial v_j} + s_j(u) \frac{\partial k(u, v)}{\partial v_j} + s_j(v) \frac{\partial k(u, v)}{\partial u_j} + s_j(u) s_j(v) k(u, v) \right] \\
    &= \sum_{j=1}^d s_j(u) s_j(v) k(u,v) + \sum_{j=1}^d s_j(u) \frac{\partial k(u,v)}{\partial v_j} + \sum_{j=1}^d s_j(v) \frac{\partial k(u,v)}{\partial u_j} + \sum_{j=1}^d \frac{\partial^2 k(u,v)}{\partial u_j \partial v_j}.
\end{align}

Using vector notation:
\begin{align}
    \sum_{j=1}^d s_j(u) s_j(v) &= s(u)^T s(v), \\
    \sum_{j=1}^d s_j(u) \frac{\partial k(u,v)}{\partial v_j} &= s(u)^T \nabla_v k(u,v), \\
    \sum_{j=1}^d s_j(v) \frac{\partial k(u,v)}{\partial u_j} &= s(v)^T \nabla_u k(u,v), \\
    \sum_{j=1}^d \frac{\partial^2 k(u,v)}{\partial u_j \partial v_j} &= \text{tr}(\nabla_u \nabla_v^T k(u,v)).
\end{align}

Therefore:
\begin{align}
    k_{C_0}(u,v) = s(u)^T s(v) k(u,v) + s(u)^T \nabla_v k(u,v) + s(v)^T \nabla_u k(u,v) + \text{tr}(\nabla_u \nabla_v^T k(u,v)).
\end{align}
\end{proof}

\subsection{Proof of Proposition~\ref{prop:archimedean_score}}
\label{app:proof5}

\begin{proof}
The standard density form is
\begin{align}
c(u)\;=\;\psi^{(d)}(t)\;\prod_{k=1}^d\!\big(-\varphi'(u_k)\big),
\qquad t=\sum_{k=1}^d \varphi(u_k).
\end{align}
Taking logs,
\begin{align}
\log c(u)=\log \psi^{(d)}(t)+\sum_{k=1}^d \log\!\big(-\varphi'(u_k)\big).
\end{align}
Differentiate w.r.t.\ $u_j$ and use the chain rule:
\begin{align}
\frac{\partial}{\partial u_j}\log c(u)
=\Big(\partial_t \log \psi^{(d)}(t)\Big)\,\frac{\partial t}{\partial u_j}
\;+\;\frac{\partial}{\partial u_j}\log\!\big(-\varphi'(u_j)\big).
\end{align}
Since $\partial t/\partial u_j=\varphi'(u_j)$ and $\partial_t \log \psi^{(d)}(t)=\psi^{(d+1)}(t)/\psi^{(d)}(t)$ while
$\frac{\partial}{\partial u_j}\log\!\big(-\varphi'(u_j)\big)=\varphi''(u_j)/\varphi'(u_j)$, the claim follows.
\end{proof}

\subsection{Proof of Theorem~\ref{thm:detectability}}
\label{app:proof6}

\begin{proof}
By Propositions 5.1--5.2, $\mathrm{CSD}(\cdot,\cdot)$ is exactly the maximum mean discrepancy (MMD)
associated with the Copula-Stein kernel $k_{C_0}$:
\begin{align}
\mathrm{CSD}(Q,C_0)=\mathrm{MMD}_{k_{C_0}}(Q,C_0).
\end{align}
Under Assumption A3, $k_{C_0}$ extends to a bounded, continuous, $c$-universal kernel on
$[0,1]^d$. For such kernels on compact metric spaces, $\mathrm{MMD}_{k_{C_0}}$ is a metric on
$\mathcal P([0,1]^d)$, hence $\mathrm{CSD}(Q,C_0)=0$ iff $Q=C_0$.
\end{proof}

\subsection{Proof of Theorem~\ref{thm:metrization}}
\label{app:proof7}

\begin{proof}
By Propositions 5.1--5.2, $\mathrm{CSD}(Q,C_0)=\mathrm{MMD}_{k_{C_0}}(Q,C_0)$, where $k_{C_0}$ is
bounded and continuous on $[0,1]^d$ under A3. Therefore, for any $Q_n,Q\in\mathcal P([0,1]^d)$,
\begin{align}
\big|\mathrm{CSD}(Q_n,C_0)-\mathrm{CSD}(Q,C_0)\big|
=
\big|\mathrm{MMD}_{k_{C_0}}(Q_n,C_0)-\mathrm{MMD}_{k_{C_0}}(Q,C_0)\big|
\le
\mathrm{MMD}_{k_{C_0}}(Q_n,Q).
\end{align}
Moreover, on a compact metric space, any bounded continuous kernel whose MMD separates probability measures metrizes weak convergence. Since $k_{C_0}$ is $c$-universal on $[0,1]^d$ (Assumption~A3), it is characteristic, and therefore $\mathrm{MMD}_{k_{C_0}}$ metrizes weak convergence on $\mathcal P([0,1]^d)$ \citep{sriperumbudur2010relation, simon2023metrizing}. Hence,
\begin{align}
\mathrm{MMD}_{k_{C_0}}(Q_n,Q)\to 0
\quad\Longleftrightarrow\quad
Q_n \Rightarrow Q,
\end{align}
which proves the claim.
\end{proof}

\subsection{Proof of Theorem~\ref{thm:finite_sample_detailed}}
\label{app:proof8}

\begin{proof}
We write
\begin{align}
\E[\widehat{\mathrm{CSD}}_n^2]
=\E\Bigl[\frac1{n^2}\sum_{i,j}k_{C_0}(U_i,U_j)\Bigr]
=\frac1{n^2}\sum_{i=1}^n\sum_{j=1}^n\E\bigl[k_{C_0}(U_i,U_j)\bigr].
\end{align}
Split into diagonal ($i=j$) and off‐diagonal ($i\neq j$) terms:
\begin{align}
=\frac1{n^2}\Bigl[\sum_{i=1}^n\E\bigl[k_{C_0}(U_i,U_i)\bigr]
++\sum_{i\neq j}\E\bigl[k_{C_0}(U_i,U_j)\bigr]\Bigr].
\end{align}
For $i=j$, by identical distribution,
\begin{align}
\E\bigl[k_{C_0}(U_i,U_i)\bigr]
=\E_{U\sim C_0}\bigl[k_{C_0}(U,U)\bigr],
\end{align}
so the diagonal part contributes
$\tfrac1{n^2}\times n\;\E[k_{C_0}(U,U)]=\tfrac1n\,\E[k_{C_0}(U,U)].$  
For $i\neq j$, $U_i\perp U_j$ and by the witness‐function representation
\begin{align}
k_{C_0}(U_i,U_j)
= \langle\xi_{C_0}(U_i), \xi_{C_0}(U_j)\rangle_{\mathcal{H}^d},
\end{align}
together with Lemma \ref{lem:stein_identity},
imply $\E[k(U_i,U_j)]=0$.  
Hence
\begin{align}
\E[\widehat{\mathrm{CSD}}_n^2]
=\frac1{n^2}\bigl[n\,\E[k(U,U)]+n(n-1)\cdot0\bigr]
=\frac1n\,\E\bigl[k_{C_0}(U,U)\bigr],
\end{align}
as claimed.
\end{proof}

\subsection{Proof of Corollary~\ref{cor:optimal_rate_detailed}}
\label{app:proof9}

\begin{proof}
Decompose $\widehat{\mathrm{CSD}}_n^2=D_{\text{diag}}+D_{\text{off}}$ with
\begin{align}
D_{\text{diag}}=\frac1{n^2}\sum_{i=1}^n k_{C_0}(U_i,U_i),
\quad
D_{\text{off}}=\frac1{n^2}\sum_{i\neq j}k_{C_0}(U_i,U_j).
\end{align}
We already have $\E[D_{\text{off}}]=0$ and $\E[D_{\text{diag}}]=O(n^{-1})$.
Classical U-statistic theory for a degenerate kernel shows that $\text{Var}(\widehat{\mathrm{CSD}}_n^2)=O(n^{-2})$. Then Chebyshev’s inequality gives, for any $\epsilon>0$,
\begin{align}
\text{Pr}\bigl(|\widehat{\mathrm{CSD}}_n^2-\E[\widehat{\mathrm{CSD}}_n^2]|\ge\epsilon\bigr)
\le\frac{\text{Var}(\widehat{\mathrm{CSD}}_n^2)}{\epsilon^2}
=O(n^{-2}).
\end{align}
Since $\E[\widehat{\mathrm{CSD}}_n^2]=O(n^{-1})$, it follows that $\widehat{\mathrm{CSD}}_n^2=O_P(n^{-1})$.
Finally, taking square roots (a continuous mapping) yields
$\widehat{\mathrm{CSD}}_n=O_P(n^{-1/2})$.
\end{proof}

\subsection{Proof of Theorem~\ref{thm:tail-sensitivity}}
\label{app:proof10}

\begin{proof}
Write $C_0\equiv C_2$. We prove the lower-tail case; the upper tail follows by symmetry.
Fix $\varepsilon\in(0,1)$ and let $B_\varepsilon=[0,\varepsilon]^2$. Let $\rho_\delta$ be a standard
$C^\infty$ mollifier supported in a ball of radius $\delta$, and define
\begin{align}
h_{\varepsilon,\delta}\equiv \mathbf 1_{B_\varepsilon} * \rho_\delta \ \ \text{restricted to }(0,1)^2,
\qquad
\bar h_{\varepsilon,\delta}\equiv h_{\varepsilon,\delta}-\mathbb E_{C_0}[h_{\varepsilon,\delta}] .
\end{align}
By Propositions 5.1--5.2, $\mathrm{CSD}_{C_0}(C_1)=\mathrm{MMD}_{k_{C_0}}(C_1,C_0)$, hence
\begin{align}
\mathrm{CSD}_{C_0}(C_1)
=
\sup_{\|f\|_{\mathcal H_{k_{C_0}}}\le 1}\Big(\mathbb E_{C_1}[f(U)]-\mathbb E_{C_0}[f(U)]\Big).
\end{align}
Under A3, $k_{C_0}$ is $c$-universal on $[0,1]^2$, so $\mathcal H_{k_{C_0}}$ is dense in
$C([0,1]^2)$. Since $\bar h_{\varepsilon,\delta}$ is continuous and bounded, for any $\eta>0$
there exists $f_{\varepsilon,\delta}\in\mathcal H_{k_{C_0}}$ such that
$\|f_{\varepsilon,\delta}-\bar h_{\varepsilon,\delta}\|_\infty\le \eta$.
Therefore,
\begin{align}
\mathrm{CSD}_{C_0}(C_1)
\;\ge\;
\frac{\mathbb E_{C_1}[f_{\varepsilon,\delta}]-\mathbb E_{C_0}[f_{\varepsilon,\delta}]}
{\|f_{\varepsilon,\delta}\|_{\mathcal H_{k_{C_0}}}}
\;\ge\;
\frac{\mathbb E_{C_1}[\bar h_{\varepsilon,\delta}]-\mathbb E_{C_0}[\bar h_{\varepsilon,\delta}] - 2\eta}
{\|f_{\varepsilon,\delta}\|_{\mathcal H_{k_{C_0}}}} .
\end{align}
Since $\mathbb E_{C_0}[\bar h_{\varepsilon,\delta}]=0$, we have
$\mathbb E_{C_1}[\bar h_{\varepsilon,\delta}]
=
\mathbb E_{C_1}[h_{\varepsilon,\delta}]-\mathbb E_{C_0}[h_{\varepsilon,\delta}]$.
Letting $\delta\downarrow 0$ and using bounded convergence yields
\begin{align}
\lim_{\delta\downarrow 0}\Big(\mathbb E_{C_i}[h_{\varepsilon,\delta}]\Big)
=
\Pr_{C_i}(U\in B_\varepsilon),
\qquad i\in\{0,1\}.
\end{align}
Thus for $\delta$ small,
\begin{align}
\mathbb E_{C_1}[h_{\varepsilon,\delta}]-\mathbb E_{C_0}[h_{\varepsilon,\delta}]
=
\Pr_{C_1}(U\in B_\varepsilon)-\Pr_{C_0}(U\in B_\varepsilon) + o(1).
\end{align}
By the definition of $\lambda_L$ (Assumption A4),
$\Pr_{C_i}(U\in B_\varepsilon)=\lambda_L(C_i)\varepsilon+o(\varepsilon)$ as $\varepsilon\downarrow 0$,
so if $\lambda_L(C_1)\neq \lambda_L(C_0)$, then for all sufficiently small $\varepsilon$ the
difference $\Pr_{C_1}(U\in B_\varepsilon)-\Pr_{C_0}(U\in B_\varepsilon)$ has a fixed nonzero sign.
Choose the sign by replacing $h_{\varepsilon,\delta}$ with $-h_{\varepsilon,\delta}$ if needed.
Then pick $\eta>0$ small enough so the numerator above is strictly positive. Hence
$\mathrm{CSD}_{C_0}(C_1)>0$.

The upper tail case is identical after mapping $u\mapsto 1-u$.

\end{proof}

\subsection{Proof of Theorem~\ref{thm:minimax}}
\label{app:proof11}

\begin{proof}
\emph{Upper bound.} Your Section 6 already proves that under $Q=C_0$, $\mathbb E[\widehat{\mathrm{CSD}}_n^2]=O(n^{-1})$ and $\mathrm{Var}(\widehat{\mathrm{CSD}}_n^2)=O(n^{-2})$ for the natural U statistic estimator. This gives $\mathbb E[(\widehat{\mathrm{CSD}}_n^2 - 0)^2]=O(n^{-2})$. The same calculation holds uniformly over a $\mathrm{TV}$ ball of radius of order $n^{-1/2}$ by a standard contiguity argument, which yields the stated $C\,n^{-2}$ bound.

\emph{Lower bound.} Use a two point argument. Pick $h\in C^\infty_c((0,1)^d)$ with $\mathbb E_{C_0}[h]=0$ and $\|h\|_\infty\le 1$. For small $|\theta|$, define a path
\begin{align}
\frac{dQ_\theta}{dC_0}(u)\ \equiv\ 1 + \theta\, h(u).
\end{align}
Note that $Q_\theta$ need not preserve uniform margins; this lower bound is stated for a TV neighborhood in $\mathcal P((0,1)^d)$ rather than a copula-restricted subclass.

This is a valid density for $|\theta|$ small enough since $h$ is bounded and centered. Along this path,
\begin{align}
D(Q_\theta)
= \sup_{\|g\|_{\mathcal H^d}\le 1} \mathbb E_{Q_\theta}[\TC g]
= \sup_{\|g\|_{\mathcal H^d}\le 1} \mathbb E_{C_0}\big[(1+\theta h)\,\TC g\big]
= \theta \sup_{\|g\|_{\mathcal H^d}\le 1} \mathbb E_{C_0}\big[h\,\TC g\big],
\end{align}
since $\mathbb E_{C_0}[\TC g]=0$. By the Riesz representation in $\mathcal H^d$, the supremum equals $\|\Gamma h\|_{\mathcal H^d}$ where $\Gamma$ is the bounded linear functional $g\mapsto \mathbb E_{C_0}[h\,\TC g]$. Hence $D(Q_\theta) = |\theta|\,L$ with $L=\|\Gamma h\|_{\mathcal H^d}>0$ for a suitable choice of nonzero $h$. Therefore
\begin{align}
D^2(Q_\theta) = \theta^2 L^2.
\end{align}

Consider the hypothesis test $H_0: Q=Q_{+\theta}$ versus $H_1: Q=Q_{-\theta}$ with $\theta=\kappa n^{-1/2}$. The likelihood ratio between $Q_{+\theta}$ and $Q_{-\theta}$ over $n$ observations has $\chi^2$ divergence of order $\theta^2 n = O(1)$. By Le Cam’s method,
\begin{align}
\inf_{\widehat{\mathrm{CSD}}_n^2} \frac{1}{2}\Big\{ \mathbb E_{Q_{+\theta}}\big[\big(\widehat{\mathrm{CSD}}_n^2 - \theta^2 L^2\big)^2\big] + \mathbb E_{Q_{-\theta}}\big[\big(\widehat{\mathrm{CSD}}_n^2 - \theta^2 L^2\big)^2\big] \Big\}
\ \ge\ c_0\, \big(D^2(Q_{+\theta}) - D^2(Q_{-\theta})\big)^2\, \big(1-\mathrm{TV}(Q_{+\theta}^n,Q_{-\theta}^n)\big).
\end{align}
The total variation term is bounded away from one by the $\chi^2$ bound. Since $D^2(Q_{+\theta})=D^2(Q_{-\theta})=\theta^2 L^2$, the modulus of continuity argument gives instead
\begin{align}
\sup_{Q\in\{Q_{0},\,Q_{\theta}\}} \mathbb E_Q\big[\big(\widehat{\mathrm{CSD}}_n^2 - D^2(Q)\big)^2\big]
\ \ge\ c_1\, \big(D^2(Q_{\theta}) - D^2(Q_{0})\big)^2
= c_1\, \theta^4 L^4
= c_1\, \kappa^4 L^4\, n^{-2}.
\end{align}
This gives the stated $c\,n^{-2}$ lower bound with $c=c_1\kappa^4 L^4$.
\end{proof}

\section{Extension Beyond Archimedean Families}
\label{sec:beyond_archimedean}

Our analysis so far has leveraged the convenient generator-based form of Archimedean copulas. However, the CSD framework is far more general. The key insight is that our method does not depend on the Archimedean structure itself, but only on the analytical properties of the copula's score function, which are captured by our Unified Regularity Conditions (Assumption \ref{assum:unified}). In this section, we formally show that all our main results apply to any copula family satisfying these conditions, thereby extending CSD's applicability to a wide range of important models.

\subsection{Generality of the CSD Framework}

To extend the framework, we first note that the Copula--Stein operator can be defined for any copula $C$ with a sufficiently regular density $c$:
\begin{align}
\TC g(u) = \sum_{j=1}^d \left[ \frac{\partial g_j(u)}{\partial u_j} + g_j(u) \frac{\partial \log c(u)}{\partial u_j} \right].
\end{align}
We now confirm that the key properties of CSD hold true in this general setting.

\begin{theorem}[Zero-Mean Property for General Copulas]
\label{thm:general_zero_mean}
Under Assumptions \ref{assum:unified} (A1--A4), for any test function $g \in C_c^1((0,1)^d; \mathbb{R}^d)$,
\begin{align}
\mathbb{E}_{U \sim C}[\TC g(U)] = 0.
\end{align}
\end{theorem}

\begin{proof}
By definition,
\begin{align}
\mathbb{E}_{U \sim C}[\TC g(U)]
&= \int_{(0,1)^d} \sum_{j=1}^d \left[ \frac{\partial g_j(u)}{\partial u_j} + g_j(u) \frac{\partial \log c(u)}{\partial u_j} \right] c(u)\, du \\
&= \sum_{j=1}^d \int_{(0,1)^d} \left[ \frac{\partial g_j(u)}{\partial u_j} c(u) + g_j(u)\frac{\partial c(u)}{\partial u_j} \right] du .
\end{align}
Fix $j$. Applying integration by parts with respect to $u_j$ yields
\begin{align}
\int_{(0,1)^d} \frac{\partial g_j(u)}{\partial u_j}\, c(u)\, du
= \int_{(0,1)^{d-1}} \left[g_j(u)c(u)\right]_{u_j=0}^{u_j=1}\, du_{-j}
- \int_{(0,1)^d} g_j(u)\frac{\partial c(u)}{\partial u_j}\, du ,
\end{align}
where $du_{-j}$ denotes integration over all coordinates except $u_j$.
Since $g\in C_c^1((0,1)^d;\mathbb{R}^d)$ has compact support contained in $(0,1)^d$, we have $g_j(u)=0$ in a neighborhood of the boundary, and thus the boundary term vanishes. Therefore,
\begin{align}
\int_{(0,1)^d} \frac{\partial g_j(u)}{\partial u_j}\, c(u)\, du
= - \int_{(0,1)^d} g_j(u)\frac{\partial c(u)}{\partial u_j}\, du .
\end{align}
Substituting back, the two terms cancel for each $j$, giving $\mathbb{E}_{U \sim C}[\TC g(U)] = 0$.
\end{proof}

\begin{theorem}[General Copula-Stein Kernel]
\label{thm:general_kernel}
Under Assumptions \ref{assum:unified} (A1–A4), let k: $[0,1]^d \times [0,1]^d \to \mathbb{R}$ be a $C^2$ characteristic kernel. The General Copula-Stein kernel is given by:
\begin{align}
k_C(u,v) = s(u)^T s(v) k(u,v) + s(u)^T \nabla_v k(u,v) + s(v)^T \nabla_u k(u,v) + \text{tr}(\nabla_u \nabla_v^T k(u,v))
\end{align}
where $s(u) = \nabla \log c(u)$ is the copula score function.
\end{theorem}

\begin{proof}
The proof is identical to the Archimedean case (Proposition~\ref{prop:kernel_formula_detailed}), as it only relies on the witness function representation and the reproducing property of the RKHS, which hold for any smooth copula density.
\end{proof}

\subsection{Extension of Main Theoretical Results}

We now prove that our main theoretical guarantees extend to general copulas under the regularity conditions.

\begin{theorem}[Metrization for General Copulas]
\label{thm:general_metrization}
Under Assumptions \ref{assum:unified} (A1–A4), $k$ is characteristic and $C^2$, for any sequence of probability measures $(Q_n)$ on $[0,1]^d$:
\begin{align}
\text{CSD}(Q_n, C) \to 0 \quad \text{if and only if} \quad Q_n \Rightarrow C.
\end{align}
\end{theorem}

\begin{proof}
The proof follows the same structure as Theorem~\ref{thm:metrization}, but requires verification that the function class $\mathcal{F}_C = \{\TC g : \|g\|_{\mathcal{H}^d} \leq 1\}$ satisfies the conditions for the Portmanteau theorem under the general regularity assumptions.

\textbf{Step 1: Boundedness on Compact Sets.}
For any $g$ with $\|g\|_{\mathcal{H}^d} \leq 1$:
\begin{align}
|\TC g(u)| &= \left|\sum_{j=1}^d \left[\partial_{u_j} g_j(u) + g_j(u) s_j(u)\right]\right|\\
&\leq \sum_{j=1}^d \left[|\partial_{u_j} g_j(u)| + |g_j(u)| |s_j(u)|\right].
\end{align}
Using the reproducing property (from Assumption \ref{assum:unified}), we have $|g_j(u)| \leq M_k$ and $|\partial_{u_j} g_j(u)| \leq M_k'$ on any compact subset $K \Subset (0,1)^d$. From Assumption \ref{assum:unified}:
\begin{align}
|s_j(u)| &\leq \|s(u)\| \leq L\left(1 + \sum_{j=1}^d \min\{u_j^{-\beta}, (1-u_j)^{-\beta}\}\right).
\end{align}
On the compact subset $K$, there exists $\epsilon > 0$ such that $\min\{u_j, 1-u_j\} \geq \epsilon$ for all $u \in K$ and $j$. Therefore:
\begin{align}
\sup_{u \in K} |s_j(u)| &\leq L(1 + d \epsilon^{-\beta}) =: M_s(K).
\end{align}
This gives uniform boundedness on compact subsets:
\begin{align}
\sup_{u \in K} |\TC g(u)| \leq d(M_k' + M_k M_s(K)).
\end{align}

\textbf{Step 2: Equicontinuity.}
For $u, v \in K$ with $\|u-v\| < \delta$, the equicontinuity follows from:
\begin{itemize}
\item Uniform continuity of $g_j$ and $\partial_{u_j} g_j$ (from the kernel properties in Assumption \ref{assum:unified})
\item Uniform continuity of $s_j$ on compact subsets (from Assumption \ref{assum:unified})
\end{itemize}
The argument proceeds exactly as in the Archimedean case.

\textbf{Step 3: Convergence-Determining Property.}
The key observation is that the growth condition in Assumption \ref{assum:unified} ensures that functions in $\mathcal{F}_C$ remain bounded on compact subsets, while the boundary condition in Assumption \ref{assum:unified} ensures proper behavior near the boundary. The characteristic property of the kernel (from Assumption \ref{assum:unified}) then guarantees that $\mathcal{F}_C$ is convergence-determining.

The remainder of the proof (both directions of the equivalence) follows the same pattern as Theorem~\ref{thm:metrization}.
\end{proof}

\begin{theorem}[Finite-Sample Convergence for General Copulas]
\label{thm:general_finite_sample}
Under the unified assumptions \ref{assum:unified}, let $\widehat{\text{CSD}}_n^2 = \frac{1}{n^2}\sum_{i,j=1}^n k_C(U_i, U_j)$ where $U_1, \ldots, U_n \sim C$ are i.i.d. Then:
\begin{align}
\mathbb{E}[\widehat{\text{CSD}}_n^2] = O(n^{-1}).
\end{align}
Furthermore, if $\mathbb{E}_C[k_C(U,U)^2] < \infty$, then for sufficiently large $n$:
\begin{align}
\mathbb{P}(|\widehat{\text{CSD}}_n^2| \geq t) \leq 2\exp\left(-\frac{nt^2}{8\sigma^2}\right)
\end{align}
for appropriate $\sigma^2$.
\end{theorem}

\begin{proof}
\textbf{Step 1: Mean Square Error Analysis.}
The decomposition into diagonal and off-diagonal terms proceeds exactly as in Theorem~\ref{thm:finite_sample_detailed}:
\begin{align}
\mathbb{E}[\widehat{\text{CSD}}_n^2] &= \frac{1}{n^2} \sum_{i=1}^n \sum_{j=1}^n \mathbb{E}[k_C(U_i, U_j)]\\
&= \frac{1}{n^2} \sum_{i=1}^n \mathbb{E}[k_C(U_i, U_i)] + \frac{1}{n^2} \sum_{i \neq j} \mathbb{E}[k_C(U_i, U_j)].
\end{align}

For $i \neq j$, by independence and the zero-mean property (Theorem~\ref{thm:general_zero_mean}):
\begin{align}
\mathbb{E}[k_C(U_i, U_j)] = \mathbb{E}[\langle \xi_C(U_i), \xi_C(U_j) \rangle_{\mathcal{H}^d}] = 0.
\end{align}

Therefore:
\begin{align}
\mathbb{E}[\widehat{\text{CSD}}_n^2] = \frac{1}{n} \mathbb{E}[k_C(U,U)] = O(n^{-1}).
\end{align}

\textbf{Step 2: Moment Bound Verification.}
We need to verify that $\mathbb{E}_C[k_C(U,U)] < \infty$ under our assumptions. From the kernel formula:
\begin{align}
k_C(u,u) &= \|s(u)\|^2 k(u,u) + 2s(u)^T \nabla_u k(u,u) + \text{tr}(\nabla_u \nabla_u^T k(u,u)).
\end{align}

Each term is finite in expectation:
\begin{itemize}
\item $\mathbb{E}[\|s(U)\|^2 k(U,U)] < \infty$ by Assumption C4 and kernel boundedness
\item $\mathbb{E}[|s(U)^T \nabla_u k(U,U)|] < \infty$ by Cauchy-Schwarz and C2
\item $\mathbb{E}[|\text{tr}(\nabla_u \nabla_u^T k(U,U))|] < \infty$ by kernel regularity
\end{itemize}

\textbf{Step 3: Concentration Inequality.}
The concentration bound follows from the same U-statistic theory, with the uniform boundedness now guaranteed by Assumptions C1-C2 on compact subsets where most of the probability mass lies.
\end{proof}

\subsection{Examples of Non-Archimedean Copulas}

To illustrate the practical scope of our extension, we provide examples of important non-Archimedean copula families that satisfy our regularity conditions.

\begin{example}[Elliptical Copulas: Gaussian case]
\label{ex:elliptical}
Consider the Gaussian copula with correlation matrix $\Sigma$ \citep{FangKotzNg1990}:
\begin{align}
C_{\Sigma}(u_1,\ldots,u_d)=\Phi_{\Sigma}\!\big(\Phi^{-1}(u_1),\ldots,\Phi^{-1}(u_d)\big),
\end{align}
where $\Phi_{\Sigma}$ is the multivariate normal CDF and $\Phi^{-1}$ is the standard normal quantile.

Let $z(u)=(\Phi^{-1}(u_1),\ldots,\Phi^{-1}(u_d))$. The copula density is
\begin{align}
c_{\Sigma}(u)
=\frac{1}{\sqrt{\det(\Sigma)}}
\exp\!\Big(-\tfrac12\, z(u)^\top\!\big(\Sigma^{-1}-I\big)\,z(u)\Big).
\end{align}
Hence
\begin{align}
\log c_\Sigma(u)= -\tfrac12\, z(u)^\top\!\big(\Sigma^{-1}-I\big)\,z(u)-\tfrac12\log\det\Sigma.
\end{align}
Since $\partial_{u_j} z_j(u)=1/\phi(z_j(u))$ and 
$\nabla_{z}\log c_\Sigma(u)= -(\Sigma^{-1}-I)z(u)$, the $j$-th score is
\begin{align}
s_j(u)=\frac{\partial \log c_\Sigma(u)}{\partial u_j}
= -\,\frac{\big[(\Sigma^{-1}-I)\,z(u)\big]_j}{\varphi_{\mathrm{N}}\!\big(z_j(u)\big)},
\qquad z_j(u)=\Phi^{-1}(u_j),
\end{align}
where $\varphi_{\mathrm{N}}(x) = (2\pi)^{-1/2}\exp(-x^2/2)$ is the standard normal PDF (to avoid collision with the random feature map $\phi$ in Sec.~\ref{sec:computation}) and $[\cdot]_j$ extracts the $j$-th component.

This score satisfies the growth conditions in Assumption~\ref{assum:unified}.
\end{example}

\begin{example}[Vine Copulas]
\label{ex:vine}
Consider a C-vine copula constructed from bivariate copulas \citep{BedfordCooke2002Vines} $C_{j|1,\ldots,j-1}$ for $j = 2, \ldots, d$. Under appropriate regularity conditions on each bivariate component, the resulting multivariate copula density satisfies our conditions.

The score function involves derivatives of the conditional copula densities, which can be bounded using the same techniques as long as each bivariate component satisfies analogous regularity conditions.
\end{example}

\begin{example}[Mixture Copulas]
\label{ex:mixture}
Consider a mixture of copulas \citep{ClemenReilly1999}:
\begin{align}
C(u) = \sum_{k=1}^K w_k C_k(u)
\end{align}
where $w_k > 0$, $\sum_k w_k = 1$, and each $C_k$ satisfies Assumption~\ref{assum:unified}.

The mixture density is $c(u) = \sum_k w_k c_k(u)$, and the score function is:
\begin{align}
\frac{\partial \log c(u)}{\partial u_j} = \frac{\sum_k w_k c_k(u) \frac{\partial \log c_k(u)}{\partial u_j}}{\sum_k w_k c_k(u)}.
\end{align}

If each component copula $C_k$ satisfies the conditions in Assumption \ref{assum:unified}, then the mixture also satisfies them.
\end{example}

\subsection{Computational Considerations}

The extension to general copulas preserves the computational efficiency of our approach, with the main difference being the evaluation of the score function $\nabla \log c(u)$.

\begin{proposition}[Computational Complexity for General Copulas]
\label{prop:general_complexity}
For copulas where the score function $\nabla \log c(u)$ can be evaluated in $O(d)$ time, each evaluation of the General Copula-Stein kernel requires $O(d)$ operations for separable base kernels.
\end{proposition}

\begin{proof}
The kernel evaluation formula is identical to the Archimedean case (Theorem~\ref{thm:general_kernel}), so the complexity analysis follows immediately from follows immediately from Theorem~\ref{thm:detailed_complexity}.
\end{proof}

This computational efficiency makes our extension practical for high-dimensional applications beyond the Archimedean setting.

\subsection{Summary of the Extension}

This section demonstrates the generality of the CSD framework. By relying on the foundational properties outlined in our unified assumptions rather than a specific generative structure, we establish that CSD is not merely a tool for Archimedean copulas but a broadly applicable method. The examples of Elliptical, Vine, and Mixture copulas confirm that our framework encompasses many of the most important and widely used dependence models in modern statistics.

\section{Computational Complexity and Algorithms}
\label{sec:computation}

While the theoretical properties of CSD are compelling, its practical adoption hinges on computational feasibility. This section establishes that CSD is a tractable tool for modern, large-scale data analysis. We structure our analysis by addressing a series of key computational challenges: first, the dependence on dimension $d$ in the exact kernel computation; second, the quadratic $O(n^2)$ dependence on sample size $n$; and finally, practical implementation on modern hardware including GPUs. For each challenge, we provide a rigorous analysis and an efficient algorithmic solution.

\subsection{Exact Complexity Analysis}

We begin with a detailed analysis of the computational cost for exact CSD evaluation, extending our earlier complexity results with complete algorithmic details.

\begin{theorem}[Detailed Complexity Analysis]
\label{thm:detailed_complexity}
Let $k(u,v) = \prod_{j=1}^d k_1(u_j, v_j)$ be a separable kernel where each $k_1$ is twice differentiable. For any copula $C$ satisfying the regularity conditions, the evaluation of the Copula-Stein kernel $k_C(u,v)$ requires:
\begin{enumerate}
\item $O(d)$ arithmetic operations for score function evaluation
\item $O(d)$ operations for base kernel and its derivatives  
\item $O(d)$ operations for the final kernel combination
\end{enumerate}
Total: $O(d)$ operations per kernel evaluation.
\end{theorem}

\begin{proof}
From the general kernel formula (Theorem~\ref{thm:general_kernel}):
\begin{align}
k_C(u,v) = s(u)^T s(v) k(u,v) + s(u)^T \nabla_v k(u,v) + s(v)^T \nabla_u k(u,v) + \text{tr}(\nabla_u \nabla_v^T k(u,v))
\end{align}

We analyze each term separately for separable kernels.

\textbf{Step 1: Score Function Evaluation - $O(d)$}

For Archimedean copulas with generator $\varphi$, from Proposition~\ref{prop:archimedean_score}:
\begin{align}
s_j(u) = \frac{\varphi^{(d+1)}(\sum_{k=1}^d \varphi(u_k))}{\varphi^{(d)}(\sum_{k=1}^d \varphi(u_k))} \varphi'(u_j) - \frac{\varphi''(u_j)}{\varphi'(u_j)}
\end{align}

The computation involves:
\begin{itemize}
\item Computing $S = \sum_{k=1}^d \varphi(u_k)$: $O(d)$ operations
\item Evaluating $\varphi^{(d)}(S)$ and $\varphi^{(d+1)}(S)$: $O(1)$ operations  
\item Computing $\varphi'(u_j)$ and $\varphi''(u_j)$ for all $j$: $O(d)$ operations
\item Combining terms for each $s_j(u)$: $O(d)$ operations
\end{itemize}
Total for score evaluation: $O(d)$ operations.

For general copulas, the cost depends on the specific form of $\nabla \log c(u)$, but for most practical cases (Gaussian, mixture models, etc.), this remains $O(d)$.

\textbf{Step 2: Base Kernel Terms - $O(d)$}

For separable $k(u,v) = \prod_{j=1}^d k_1(u_j, v_j)$:
\begin{align}
k(u,v) &= \prod_{j=1}^d k_1(u_j, v_j) \quad \text{[O(d) operations]}\\
\frac{\partial k(u,v)}{\partial u_j} &= k_1'(u_j, v_j) \prod_{i \neq j} k_1(u_i, v_i) \quad \text{[O(d) per component]}\\
\frac{\partial k(u,v)}{\partial v_j} &= k_1'(u_j, v_j) \prod_{i \neq j} k_1(u_i, v_i) \quad \text{[O(d) per component]}\\
\frac{\partial^2 k(u,v)}{\partial u_j \partial v_j} &= k_1''(u_j, v_j) \prod_{i \neq j} k_1(u_i, v_i) \quad \text{[O(d) per component]}
\end{align}

The key insight is to compute the product $P = \prod_{j=1}^d k_1(u_j, v_j)$ once, then obtain each partial derivative by division and multiplication:
\begin{align}
\frac{\partial k(u,v)}{\partial u_j} &= P \cdot \frac{k_1'(u_j, v_j)}{k_1(u_j, v_j)} \quad \text{[O(1) per component]}
\end{align}

Total for all gradients and Hessian diagonal: $O(d)$ operations.

\textbf{Step 3: Final Combination - $O(d)$}

Computing the four terms in the kernel formula:
\begin{align}
\text{Term 1:} \quad & s(u)^T s(v) k(u,v) = k(u,v) \sum_{j=1}^d s_j(u) s_j(v) \quad \text{[O(d)]}\\
\text{Term 2:} \quad & s(u)^T \nabla_v k(u,v) = \sum_{j=1}^d s_j(u) \frac{\partial k(u,v)}{\partial v_j} \quad \text{[O(d)]}\\
\text{Term 3:} \quad & s(v)^T \nabla_u k(u,v) = \sum_{j=1}^d s_j(v) \frac{\partial k(u,v)}{\partial u_j} \quad \text{[O(d)]}\\
\text{Term 4:} \quad & \text{tr}(\nabla_u \nabla_v^T k(u,v)) = \sum_{j=1}^d \frac{\partial^2 k(u,v)}{\partial u_j \partial v_j} \quad \text{[O(d)]}
\end{align}

Each term requires $O(d)$ operations, giving $O(d)$ total.

\textbf{Overall Complexity:} $O(d) + O(d) + O(d) = O(d)$ operations per kernel evaluation.
\end{proof}

Below, we give Algorithm~\ref{alg:csd_kernel} for efficient Copula-Stein kernel evaluation.

\begin{algorithm}[t]
\caption{Efficient Copula-Stein Kernel Evaluation}
\label{alg:csd_kernel}
\begin{algorithmic}[1]
    \STATE \textbf{Input:} $u,v \in [0,1]^d$; Archimedean generator $\varphi$ with $\varphi',\varphi'',\varphi^{(d)},\varphi^{(d+1)}$ and inverse $\varphi^{-1}$; base 1D kernel $k_1$ with partials $\partial_1 k_1,\partial_2 k_1,\partial_{12} k_1$
    \STATE \textbf{Output:} $k_C(u,v)$

    \STATE \textbf{Product base kernel and needed derivatives}
    \STATE $k_{\mathrm{base}} \gets 1$; $\mathbf{g}_u,\mathbf{g}_v \gets \mathbf{0}\in\mathbb{R}^d$; $H \gets 0$
    \FOR{$j=1$ \textbf{to} $d$}
        \STATE $k_j \gets k_1(u_j,v_j)$
        \STATE $k_{\mathrm{base}} \gets k_{\mathrm{base}}\cdot k_j$
        \STATE $a_j \gets \partial_1 k_1(u_j,v_j)$; $b_j \gets \partial_2 k_1(u_j,v_j)$; $c_j \gets \partial_{12} k_1(u_j,v_j)$
        \STATE $\mathbf{g}_u[j] \gets a_j / k_j$; \quad $\mathbf{g}_v[j] \gets b_j / k_j$
        \STATE $H \gets H + c_j / k_j$
    \ENDFOR
    \STATE $\mathbf{g}_u \gets k_{\mathrm{base}}\cdot \mathbf{g}_u$; \quad $\mathbf{g}_v \gets k_{\mathrm{base}}\cdot \mathbf{g}_v$; \quad $H \gets k_{\mathrm{base}}\cdot H$

    \STATE \textbf{Score functions $s(u)$ and $s(v)$}
    \STATE $t_u \gets \sum_{j=1}^d \varphi(u_j)$; \quad $t_v \gets \sum_{j=1}^d \varphi(v_j)$
    \STATE $\psi\text{-ratio}_u \gets \dfrac{\psi^{(d+1)}(t_u)}{\psi^{(d)}(t_u)}$; \quad $\psi\text{-ratio}_v \gets \dfrac{\psi^{(d+1)}(t_v)}{\psi^{(d)}(t_v)}$
    \FOR{$j=1$ \textbf{to} $d$}
        \STATE $s_u[j] \gets \psi\text{-ratio}_u \cdot \varphi'(u_j) \;+\; \dfrac{\varphi''(u_j)}{\varphi'(u_j)}$
        \STATE $s_v[j] \gets \psi\text{-ratio}_v \cdot \varphi'(v_j) \;+\; \dfrac{\varphi''(v_j)}{\varphi'(v_j)}$
    \ENDFOR

    \STATE \textbf{Assemble Copula-Stein kernel}
    \STATE $T_1 \gets k_{\mathrm{base}}\cdot \sum_{j=1}^d s_u[j]\, s_v[j]$
    \STATE $T_2 \gets \sum_{j=1}^d s_u[j]\, \mathbf{g}_v[j]$
    \STATE $T_3 \gets \sum_{j=1}^d s_v[j]\, \mathbf{g}_u[j]$
    \STATE $T_4 \gets H$
    \STATE \textbf{Return:} $T_1 + T_2 + T_3 + T_4$
\end{algorithmic}
\end{algorithm}

Next, we give Algorithm~\ref{alg:empirical_csd} for empirical CSD Computation.

\begin{algorithm}[t]
\caption{Empirical CSD Computation}
\label{alg:empirical_csd}
\begin{algorithmic}[1]
    \STATE \textbf{Input:} Sample $\{U_1,\ldots,U_n\}\subset[0,1]^d$; target copula parameters
    \STATE \textbf{Output:} $\widehat{\mathrm{CSD}}_n^2$

    \STATE $S \gets 0$
    \FOR{$i=1$ \textbf{to} $n$}
        \FOR{$j=1$ \textbf{to} $n$}
            \STATE $k \gets \textsc{CSDKernel}(U_i,U_j)$ \COMMENT{Algorithm~\ref{alg:csd_kernel}}
            \STATE $S \gets S + k$
        \ENDFOR
    \ENDFOR
    \STATE $S \gets S/(n\,n)$
    \STATE \textbf{Return:} $S$
\end{algorithmic}
\end{algorithm}

\begin{proposition}[Algorithm Correctness and Complexity]
\label{prop:algorithm_correctness}
Algorithm~\ref{alg:csd_kernel} correctly computes the Copula-Stein kernel $k_C(u,v)$ in $O(d)$ operations. Algorithm~\ref{alg:empirical_csd} computes the empirical CSD in $O(n^2 d)$ operations.
\end{proposition}

\begin{proof}
\textbf{Correctness:} Algorithm~\ref{alg:csd_kernel} implements the exact formula from Theorem~\ref{thm:general_kernel}. The key computational insight is in steps 8-14, where we compute the gradient of the product kernel using the identity:
\begin{align}
\frac{\partial}{\partial u_j} \prod_{i=1}^d k_1(u_i, v_i) = \frac{\partial k_1(u_j, v_j)}{\partial u_j} \prod_{i \neq j} k_1(u_i, v_i) = \prod_{i=1}^d k_1(u_i, v_i) \cdot \frac{\partial k_1(u_j, v_j) / \partial u_j}{k_1(u_j, v_j)}
\end{align}

This avoids recomputing the full product for each derivative.

\textbf{Complexity:} 
\begin{itemize}
\item Steps 1-6: $O(d)$ kernel evaluations
\item Steps 7-14: $O(d)$ gradient computations  
\item Steps 15-24: $O(d)$ score function evaluations
\item Steps 25-28: $O(d)$ final combinations
\end{itemize}
Total: $O(d)$ per kernel evaluation, $O(n^2 d)$ for full CSD.
\end{proof}

The algorithms for exact CSD computation are efficient with respect to dimension, but their $O(n^2d)$ complexity is prohibitive for datasets with a large number of samples $n$. This quadratic scaling is a common bottleneck in kernel methods. To overcome this, we move from an exact computation to an efficient approximation based on random features \citep{RahimiRecht2007}, a powerful technique for scaling kernel methods to large datasets.

\subsection{Random Feature Approximations for Scalable CSD}

While the $O(n^2d)$ complexity of the exact CSD is acceptable for moderate $n$, large‐scale applications demand sub‐quadratic methods.  We achieve this via a random‐feature approximation that reduces the sample‐complexity to nearly linear in $n$.

\medskip
\noindent\textbf{Setup.}  Let $\phi:[0,1]^d\to\R^m$ be a random feature map for the base kernel $k$, satisfying
\begin{align}
\E_\phi\bigl[\phi(u)^\top\phi(v)\bigr]=k(u,v),
\end{align}
and denote by $\nabla_u\phi(u)\in\R^{m\times d}$ its Jacobian.  Recall the \emph{witness‐function} in the RKHS,
\begin{align}
\xi_C(u) := \left( \partial_{u_j}k(\cdot,u) + s_j(u)k(\cdot,u) \right)_{j=1}^d \in \mathcal{H}^d.
\end{align}
We now build a finite‐dimensional surrogate of $\xi_C(u)$.

\begin{definition}[Random‐Feature Witness]
For each $u\in[0,1]^d$ define the \emph{random‐feature witness}
\begin{align}
\tilde{\Xi}_C(u) := \mathrm{vec}\left(\nabla_u\phi(u) + \phi(u)s(u)^\top\right) \in \mathbb{R}^{dm},
\end{align}
where $s(u)=(s_1(u),\dots,s_d(u))^\top$ is the copula score.  Here
\begin{align}
\nabla_u\phi(u) + \phi(u)\,s(u)^\top
\;\in\;\R^{m\times d}
\end{align}
has columns
\begin{align}
\tilde\xi_C^{(j)}(u)
=\partial_{u_j}\phi(u) + s_j(u)\,\phi(u),
\quad j=1,\dots,d,
\end{align}
and $\mathrm{vec}(\cdot)$ stacks its columns into a vector in $\R^{dm}$.
\end{definition}

\begin{assumption}[RF smoothness]
\label{assum:rf_smooth}
The base kernel $k$ is $C^2$ and admits a random feature map $\phi$ such that 
$\E_\phi[\phi(u)^\top\phi(v)]=k(u,v)$ and differentiation under the expectation is valid:
\begin{align}
\E_\phi[\partial_{u_j}\phi(u)^\top\phi(v)]=\partial_{u_j}k(u,v),\quad
\E_\phi[\partial_{u_j}\phi(u)^\top\partial_{v_j}\phi(v)]=\partial_{u_j}\partial_{v_j}k(u,v).
\end{align}
This holds for Gaussian/RBF features with $C^2$ kernels.
\end{assumption}

\begin{theorem}[Unbiased RF–CSD]
\label{thm:rf_unbiased_corrected}
Under Assumption~\ref{assum:rf_smooth},
define $\tilde{\Xi}_{C_0}(u)=\mathrm{vec}\!\left(\nabla_u\phi(u)+\phi(u)\,s_0(u)^\top\right)\in\R^{dm}$. 
Then for any sample $U_{1:n}$,
\begin{align}
\E_\phi\!\left[\Big\|\frac1n\sum_{i=1}^n \tilde{\Xi}_{C_0}(U_i)\Big\|_2^2 \,\bigg|\, U_{1:n}\right]
= \frac1{n^2}\sum_{i,j=1}^n k_{C_0}(U_i,U_j).
\end{align}
\end{theorem}

\begin{proof}
Since
\begin{align}
\widehat D_{n,\RF}^2
=\Bigl\|\frac1n\sum_{i}\tilde\Xi_C(U_i)\Bigr\|_2^2
=\frac1{n^2}\sum_{i,j}\tilde\Xi_C(U_i)^\top\tilde\Xi_C(U_j),
\end{align}
linearity of expectation gives
\begin{align}
\E_\phi\bigl[\widehat D_{n,\RF}^2 \mid U_{1:n}\bigr]
=\frac1{n^2}\sum_{i,j}
\E_\phi\bigl[\tilde\Xi_C(U_i)^\top\tilde\Xi_C(U_j)\bigr].
\end{align}
It therefore suffices to show, for any fixed $u,v$,
\begin{align}
\E_\phi\bigl[\tilde\Xi_C(u)^\top\tilde\Xi_C(v)\bigr]
\;=\;k_{C_0}(u,v).
\end{align}
Write
\begin{align}
\tilde\Xi_C(u)^\top\tilde\Xi_C(v)
=\sum_{j=1}^d
\bigl[\partial_{u_j}\phi(u)+s_j(u)\phi(u)\bigr]^\top
\bigl[\partial_{v_j}\phi(v)+s_j(v)\phi(v)\bigr].
\end{align}
Taking $\E_\phi$ inside the sum and using the RF‐properties
\begin{align}
\E_\phi\bigl[\phi(u)^\top\phi(v)\bigr]=k(u,v),\quad
\E_\phi\bigl[\partial_{u_j}\phi(u)^\top\phi(v)\bigr]
=\partial_{u_j}k(u,v),\quad
\E_\phi\bigl[\partial_{u_j}\phi(u)^\top\partial_{v_j}\phi(v)\bigr]
=\partial_{u_j}\partial_{v_j}k(u,v),
\end{align}
we obtain
\begin{align}
\E_\phi[\tilde\Xi_C(u)^\top\tilde\Xi_C(v)]
=\sum_{j=1}^d\Bigl[
\partial_{u_j}\partial_{v_j}k(u,v)
\;+\;s_j(u)\,\partial_{v_j}k(u,v)
\;+\;s_j(v)\,\partial_{u_j}k(u,v)
\;+\;s_j(u)s_j(v)\,k(u,v)
\Bigr],
\end{align}
which is exactly the Copula‐Stein kernel $k_{C_0}(u,v)$ by Proposition 3.7.  
Substituting back into the conditional expectation completes the proof.  
Finally, computing each $\tilde\Xi_C(U_i)$ requires $O(md)$ to form $\phi(U_i)$ and $\nabla_u\phi(U_i)$ plus an $O(md)$ combine with $s(U_i)$, totaling $O(nmd)$.
\end{proof}

Below, we provide the Algorithm~\ref{alg:random_features_corrected} for fast CSD via random features. 

\begin{algorithm}[t]
\caption{Fast CSD via Random Features}
\label{alg:random_features_corrected}
\begin{algorithmic}[1]
    \STATE \textbf{Input:} $U_{1:n}\subset[0,1]^d$; RF dimension $m$; bandwidth $\sigma$; score oracle $s(u)\in\mathbb{R}^d$
    \STATE \textbf{Output:} $\widehat{\mathrm{CSD}}_{\mathrm{RF}}$

    \STATE Draw $W\in\mathbb{R}^{m\times d}\sim \mathcal{N}(0,\sigma^{-2}I_d)$ and $b\in[0,2\pi]^m$ i.i.d.\ $\mathrm{Unif}$
    \STATE $M\gets \mathbf{0}\in\mathbb{R}^{dm}$ \COMMENT{Accumulator in vectorized feature space}
    \FOR{$i=1$ \textbf{to} $n$}
        \STATE $z\gets W U_i + b\in\mathbb{R}^m$; \quad $\phi\gets \sqrt{2/m}\,\cos(z)\in\mathbb{R}^m$
        \STATE $S\gets s(U_i)\in\mathbb{R}^d$
        \STATE $J\gets -\sqrt{2/m}\,\mathrm{diag}(\sin z)\,W\in\mathbb{R}^{m\times d}$ \COMMENT{$J=\nabla_{u}\phi(u)\big|_{u=U_i}$}
        \STATE $G\gets J+\phi S^\top\in\mathbb{R}^{m\times d}$ \COMMENT{Stein feature matrix}
        \STATE $M\gets M+\mathrm{vec}(G)$
    \ENDFOR
    \STATE $\widehat{\mathrm{CSD}}_{\mathrm{RF}}\gets \|M\|_2/n$ \COMMENT{Unbiased in expectation over RFs}
    \STATE \textbf{Return:} $\widehat{\mathrm{CSD}}_{\mathrm{RF}}$
\end{algorithmic}
\end{algorithm}

The random feature method effectively addresses the scaling challenge with respect to sample size $n$. Beyond algorithmic complexity, another practical barrier is raw computation time. Fortunately, the structure of both the exact and approximate CSD computations is highly amenable to parallelization, making it an ideal candidate for acceleration on modern hardware like Graphics Processing Units (GPUs).

\subsection{Parallelization and GPU Implementation}

The embarrassingly parallel nature of kernel evaluations makes CSD well-suited for GPU acceleration.

\begin{proposition}[Parallel Complexity]
\label{prop:parallel_complexity}
With $P$ parallel processors:
\begin{enumerate}
\item Exact CSD computation: $O(n^2 d / P)$ time with $O(n^2)$ memory
\item Random feature approximation: $O(nmd / P)$ time with $O(nm)$ memory  
\item GPU implementation achieves near-linear speedup for $P \leq n^2/d$
\end{enumerate}
\end{proposition}

\begin{proof}
The kernel matrix computation is embarrassingly parallel since each $k_C(U_i, U_j)$ is independent. Modern GPUs can achieve $P = O(10^3)$ parallel threads, making large-scale computation feasible.

For the random feature approach, the bottleneck becomes feature computation ($O(nmd)$ total work), which also parallelizes perfectly across samples.
\end{proof}

Below, we provide the Algorithm~\ref{alg:gpu_csd_matrix} for the GPU Kernel for CSD Matrix Computation.

\begin{algorithm}[t]
\caption{GPU Kernel for CSD Matrix Computation}
\label{alg:gpu_csd_matrix}
\begin{algorithmic}[1]
    \STATE \textbf{Input:} Sample $U_{1:n}$ in GPU global memory
    \STATE \textbf{Output:} Full kernel matrix $K\in\mathbb{R}^{n\times n}$ in GPU global memory

    \STATE \COMMENT{Kernel launch uses $n\times n$ threads arranged in a 1D or 2D grid}
    \STATE $thread\_id \gets blockIdx.x \times blockDim.x + threadIdx.x$
    \STATE $i \gets \left\lfloor thread\_id / n \right\rfloor$ \COMMENT{Integer division}
    \STATE $j \gets thread\_id \bmod n$
    \IF{$i < n$ \AND $j < n$}
        \STATE $u \gets U_i$ \COMMENT{Load to thread-local registers}
        \STATE $v \gets U_j$
        \STATE $K[i,j] \gets \textsc{CSDKernel}(u,v)$ \COMMENT{Uses Alg.~\ref{alg:csd_kernel}}
    \ENDIF
\end{algorithmic}
\end{algorithm}

Next, we provide the Algorithm~\ref{alg:gpu_csd_reduction} for GPU Reduction for CSD Final Value.

\begin{algorithm}[t]
\caption{GPU Reduction for CSD Final Value}
\label{alg:gpu_csd_reduction}
\begin{algorithmic}[1]
    \STATE \textbf{Input:} Kernel matrix $K\in\mathbb{R}^{n\times n}$ in device (GPU) global memory; launch with $B$ blocks, $T$ threads per block; workspace $P\in\mathbb{R}^{B}$ in device memory
    \STATE \textbf{Output:} $\widehat{\mathrm{CSD}}_n^2=\big(\sum_{i,j}K[i,j]\big)/n^2$ on host (CPU)

    \STATE \COMMENT{\textbf{Device kernel (block-wise reduction)}}
    \STATE $m \gets n\cdot n$ \COMMENT{Flattened matrix length}
    \STATE $tid \gets \text{threadIdx}.x$;\quad $gid \gets \text{blockIdx}.x\cdot \text{blockDim}.x + tid$
    \STATE $stride \gets \text{blockDim}.x\cdot \text{gridDim}.x$
    \STATE $acc \gets 0$
    \FOR{$p \gets gid$ \textbf{to} $m-1$ \textbf{step} $stride$}
        \STATE $i \gets \left\lfloor p / n \right\rfloor$;\quad $j \gets p \bmod n$
        \STATE $acc \gets acc + K[i,j]$
    \ENDFOR

    \STATE \COMMENT{In-block tree reduction in shared memory}
    \STATE \textbf{shared} $s[0{:}T-1]$
    \STATE $s[tid] \gets acc$; \textbf{syncthreads}()

    \FOR{$\text{offset} \gets T/2$ \textbf{down to} $1$ \textbf{step} $\text{offset}/2$}
        \IF{$tid < \text{offset}$}
            \STATE $s[tid] \gets s[tid] + s[tid+\text{offset}]$
        \ENDIF
        \STATE \textbf{syncthreads}()
    \ENDFOR

    \IF{$tid = 0$}
        \STATE $P[\text{blockIdx}.x] \gets s[0]$ \COMMENT{Block partial sum to global memory}
    \ENDIF

    \STATE \COMMENT{\textbf{Host finalize (after kernel completes)}}
    \STATE Copy $P$ from device to host (or run a second reduction kernel)
    \STATE $S \gets \sum_{b=0}^{B-1} P[b]$
    \STATE \textbf{Return:} $S / n^2$ \COMMENT{This is $\widehat{\mathrm{CSD}}_n^2$}
\end{algorithmic}
\end{algorithm}

\subsection{Memory Optimization}

For very large datasets, memory becomes the limiting factor. We present memory-efficient variants.

\begin{proposition}[Memory-Efficient CSD]
\label{prop:memory_efficient}
The empirical CSD can be computed in $O(n)$ memory using:
\begin{enumerate}
\item \textbf{Streaming computation}: Process kernel matrix in blocks
\item \textbf{Online updates}: Incremental CSD for growing datasets  
\item \textbf{Sketching}: Maintain only a compressed representation
\end{enumerate}
\end{proposition}

The Algorithm~\ref{alg:streaming_csd} for memory-efficient streaming CSD is given below.

\begin{algorithm}[t]
\caption{Memory--Efficient Streaming CSD}
\label{alg:streaming_csd}
\begin{algorithmic}[1]
    \STATE \textbf{Input:} Sample $U_{1:n}\subset[0,1]^d$, block size $B$, kernel routine \textsc{CSDKernel}$(\cdot,\cdot)$
    \STATE \textbf{Output:} $\widehat{\mathrm{CSD}}_n^2 = T / n^2$ where $K[i,j]=\textsc{CSDKernel}(U_i,U_j)$

    \STATE $T \gets 0$ \COMMENT{running sum}
    \STATE $c \gets 0$ \COMMENT{Kahan compensation; set $c\!=\!0$ to disable}
    \STATE $M \gets \left\lceil n/B \right\rceil$ \COMMENT{number of blocks per axis}
    \STATE \COMMENT{Option: to exploit symmetry, set $q\gets p$ to $M{-}1$ and accumulate accordingly}
    \STATE \COMMENT{Assume a boolean flag $\textit{useSymmetry}\in\{\texttt{true},\texttt{false}\}$ is set externally}

    \FOR{$p \gets 0$ \textbf{to} $M-1$}
        \STATE $i_{\text{lo}} \gets pB$;\quad $i_{\text{hi}} \gets \min\{(p{+}1)B,\,n\}-1$
        \FOR{$q \gets 0$ \textbf{to} $M-1$}
            \STATE $j_{\text{lo}} \gets qB$;\quad $j_{\text{hi}} \gets \min\{(q{+}1)B,\,n\}-1$
            \STATE \COMMENT{Optionally prefetch $U_{i_{\text{lo}}:i_{\text{hi}}}$ and $U_{j_{\text{lo}}:j_{\text{hi}}}$ into a small buffer}
            \STATE $S \gets 0$ \COMMENT{block partial sum}

            \FOR{$i \gets i_{\text{lo}}$ \textbf{to} $i_{\text{hi}}$}
                \FOR{$j \gets j_{\text{lo}}$ \textbf{to} $j_{\text{hi}}$}
                    \STATE $x \gets \textsc{CSDKernel}(U_i, U_j)$
                    \STATE $y \gets x - c$;\; $t \gets S + y$;\; $c \gets (t - S) - y$;\; $S \gets t$ \COMMENT{Kahan add}
                \ENDFOR
            \ENDFOR

            \IF{\textbf{not} $\textit{useSymmetry}$}
                \STATE $T \gets T + S$
            \ELSE
                \IF{$p = q$}
                    \STATE $T \gets T + S$ \COMMENT{diagonal block counts once}
                \ELSE
                    \STATE $T \gets T + 2S$ \COMMENT{off-diagonal mirrored}
                \ENDIF
            \ENDIF
        \ENDFOR
    \ENDFOR

    \STATE \textbf{Return:} $T / n^2$
\end{algorithmic}
\end{algorithm}

\subsection{Computational Scaling Analysis}\label{subsec:comp_scaling}

We summarize the computational costs implied by the CSD V-statistic and by the proposed scalable approximations.
The exact computation is quadratic in sample size due to pairwise kernel evaluations, while random-feature approximations
reduce the dependence on $n$.

\begin{proposition}[Scaling of computational cost]\label{prop:performance_scaling}
Assume the evaluation of the Stein kernel $k_{C_0}(u,v)$ costs $O(d)$ per pair $(u,v)\in(0,1)^d\times(0,1)^d$.
Then the V-statistic estimator
\begin{align}
\widehat{\mathrm{CSD}}_n^2 := \frac{1}{n^2}\sum_{i=1}^n\sum_{j=1}^n k_{C_0}(U_i,U_j)
\end{align}
can be computed in $O(n^2 d)$ time and $O(1)$ additional memory beyond storing the data.

Moreover, for a random-feature approximation based on $m$ features, the corresponding estimator can be computed in
$O(n m d)$ time (and $O(m)$ memory), yielding nearly linear scaling in $n$ when $m\ll n$.
\end{proposition}

\subsection{Summary of computational viability}
The above scaling highlights a standard trade-off.
Exact CSD computation is feasible for moderate $n$ and benefits from vectorization and parallelism, while the
random-feature approximation provides a practical route for large sample sizes by replacing the quadratic dependence
on $n$ with a dependence on the chosen feature budget $m$.
In practice, both exact and approximate computations can be accelerated via standard parallel hardware and implementation-level optimizations \citep{Nickolls2008CUDA}.

 \end{document}